\documentclass{article}

\usepackage[final]{neurips_2022}

\usepackage[utf8]{inputenc} % allow utf-8 input
\usepackage[T1]{fontenc}    % use 8-bit T1 fonts
\usepackage{hyperref}       % hyperlinks
\usepackage{url}            % simple URL typesetting
\usepackage{booktabs}       % professional-quality tables
\usepackage{amsfonts}       % blackboard math symbols
\usepackage{nicefrac}       % compact symbols for 1/2, etc.
\usepackage{microtype}      % microtypography
\usepackage{xcolor}         % colors
\usepackage{multirow}
\usepackage{amsfonts, amsmath, amssymb,mathtools,bbm,bm}
\usepackage[inline]{enumitem}
\usepackage{algorithm,algorithmic}
\usepackage{graphicx}
\usepackage{subfig}

\newcommand{\norm}[1]{\|#1\|}
\newcommand{\one}{\mathbbm{1}}

\usepackage{amsthm}
\theoremstyle{plain}

\newtheorem{proposition}{Proposition}
\newtheorem{corollary}{Corollary}
\newtheorem{lemma}{Lemma}
\newtheorem{fact}{Fact}
\newtheorem{definition}{Definition}

\title{KERPLE: Kernelized Relative Positional Embedding for Length Extrapolation}

\author{%
  Ta-Chung Chi$^*$
  \\
  Carnegie Mellon University\\
  \texttt{tachungc@andrew.cmu.edu} \\
  \And
  Ting-Han Fan$^*$ \\
  Princeton University \\
  \texttt{tinghanf@princeton.edu} \\
  \AND
  Peter J. Ramadge \\
  Princeton University\\
  \texttt{ramadge@princeton.edu}\\ 
  \And
  Alexander I. Rudnicky \\
  Carnegie Mellon University \\
  \texttt{air@cs.cmu.edu} \\
}

\begin{document}

\begin{NoHyper}
\def\thefootnote{$^*$}\footnotetext{Equal contribution}
\end{NoHyper}

\maketitle

\begin{abstract}
Relative positional embeddings (RPE) have received considerable attention since RPEs effectively model the relative distance among tokens and enable length extrapolation. We propose KERPLE, a framework that generalizes relative position embedding for extrapolation by kernelizing positional differences. We achieve this goal using conditionally positive definite (CPD) kernels, a class of functions known for generalizing distance metrics. To maintain the inner product interpretation of self-attention, we show that a CPD kernel can be transformed into a PD kernel by adding a constant offset. This offset is implicitly absorbed in the Softmax normalization during self-attention. The diversity of CPD kernels allows us to derive various RPEs that enable length extrapolation in a principled way. Experiments demonstrate that the logarithmic variant achieves excellent extrapolation performance on three large language modeling datasets. Our implementation and pretrained checkpoints are released at~\url{https://github.com/chijames/KERPLE.git}.
\end{abstract}

\section{Introduction}
Transformer-based models have excelled in various natural language processing tasks such as chatbot~\citep{roller2021recipes}, code completion~\citep{chen2021codex}, and paper abstract summarization~\citep{zhang2020pegasus}.
These sequence modeling tasks often require the model to operate well on significantly longer text sequences than the fixed maximum length $L$ used at training time. Training (or retraining) the model using a substantially larger value of $L$ is often infeasible since the transformer training cost is $O(L^2)$. Hence, one desires a transformer that continues to perform well on longer sequences than those used during training; i.e., perform length extrapolation at inference time. Most transformer designs do not have this property \citep{press2022train}. While recent work on absolute positional embeddings demonstrated the extrapolation ability \citep{kiyono2021shiftedAbs,Likhomanenko2021CAPE}, it is believed that relative positional embeddings are more robust to input length change \citep{Likhomanenko2021CAPE}, for example, ALiBi \citep{press2022train} and T5 \citep{raffel2019exploring}. Hence, we are motivated to study the inner workings of relative positional embeddings.

%The above observations motivate us to dive deeper into the inner workings of positional embeddings and examine critical factors that ensure length extrapolation. We explain the details of our approach in Sections \ref{sec:cpd} and \ref{sec:krpe}. Here we highlight some key elements of the approach.
Relative positional embeddings (RPE) encode the idea of shift-invariance: for any shift $p$, $(m+p)-(n+p)=m-n$. It is often added directly to the self-attention matrix before Softmax normalization \citep{chen2021simple}. Inspired by shift-invariance and the ability of a kernel to define a similarity function, there have been studies on shift-invariant kernels for RPE~\citep{wennberg2021case} with a focus on Gaussian kernel. However, in our preliminary experiments, the Gaussian kernel demonstrates limited length extrapolation ability (see Appendix~\ref{appendix:experi_gauss}). Hence, a distinct class of shift-invariant kernels is needed to achieve adequate length extrapolation.

To this end, we note a set of well-established conditionally positive definite (CPD) kernels suitable for modeling distance metrics~\citep{Schplkopf2000cpd}. However, CPD kernels do not conform to an inner product. We can remedy this issue by transforming a CPD kernel into a PD kernel by adding a sufficiently large constant. This constant offset is subsequently absorbed implicitly in the Softmax normalization (see the discussion below Eq.~\eqref{eq:krpe-deriv}).
For example, ALiBi implicitly admits a PD kernel of the form $c-|m-n|$ (see the end of section~\ref{sec:krpe}), which is reduced to a CPD kernel $-|m-n|$. The CPD kernel and Softmax normalization combination opens the door to a sea of possible CPD kernels. We investigate structures from this class that exhibit a strong length extrapolation ability, like ALiBi.

\begin{figure*}[t]
\centering
\caption{\textbf{The 3-Para-Log Variant of Our KERPLE Framework.} $a$, $b$, and $p$ are learnable parameters in each attention head shared across layers. Since \# of heads is $H$, there are $3\cdot H$ learnable parameters. The learnable parameters are trained with length-3 sequences. At the inference time, the last row (in dashed squares) becomes active, and the model extrapolates to length-4 sequences. Note we focus on causal language modeling following ALiBi, so the matrices are triangular.}
\includegraphics[width=0.6\textwidth]{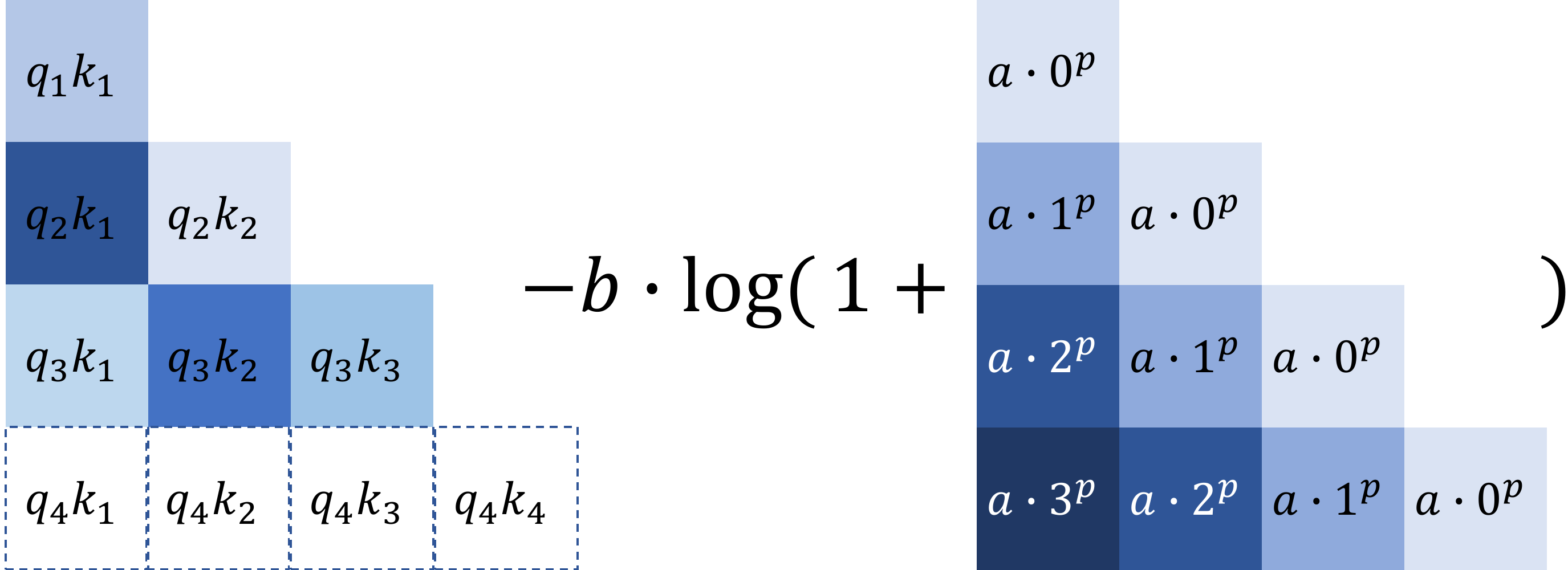}
\label{fig:illustration}
\end{figure*}

Our main result is a framework for \textbf{KE}rnelize \textbf{R}elative \textbf{P}ositional Embedding for \textbf{L}ength \textbf{E}xtrapolation (\textbf{KERPLE}). The framework elucidates key principles that encourage the length extrapolation property. We show that ALiBi is a particular instance within our framework. Our subsequent experiments suggest that the proposed method yields better length extrapolation on large datasets such as OpenWebText2, GitHub, and ArXiv.

%Based on the above discussion, we present \textbf{KERPLE}, a framework to \textbf{KE}rnelize \textbf{R}elative \textbf{P}ositional Embedding for \textbf{L}ength \textbf{E}xtrapolation. We focus on kernels that enable length extrapolation and show that ALiBi is a special case in our framework. Experiments suggest that our method achieves better length extrapolation results on large datasets such as OpenWebText2, GitHub, and ArXiv.

\section{Background and Related Work}
\subsection{Preliminary}
Let $\{w_m\}_{m=1}^L$ be the input tokens to a transformer model, where $L$ is the total number of tokens. Each $w_m$ is a scalar and is used to index the embedding vector $\bm e_m \in\mathbb{R}^d$ as the input to the transformer. A transformer converts each $\bm e_m$ into query, key, and value vectors in $\mathbb{R}^d$: $\bm q_m=\bm W_q\bm e_m$, $\bm k_m=\bm W_k \bm e_m$, $\bm v_m=\bm W_v \bm e_m$, where $\bm W_q$, $\bm W_k$, $\bm W_v\in\mathbb{R}^{d\times d}$ are learnable matrices. Then, the self-attention module computes the scaled attention scores and generates the output vector $\bm o_m$ at position $m$ as: $$a_{m,n}=\frac{\exp(\bm q_m^\top \bm k_n/\sqrt{d})}{\sum_{i=1}^L \exp(\bm q_m^\top \bm k_i/\sqrt{d})},\quad \bm o_m = \sum_{n=1}^L a_{m,n}\bm v_n.$$
Since the operation is position-agnostic, it is believed that positional information helps model token interactions \citep{vaswani2017attention}, which we survey in the next subsection.

\subsection{Positional Embedding}
\paragraph{Absolute.} Absolute positional embeddings assign a positional vector $\bm p_m$ to each position $m$ and adds $\bm p_m$ to the embedding vector $\bm e_m$. The very first version of which is the predefined sinusoidal function~\citep{vaswani2017attention}. Followed by the success of BERT~\citep{devlin2018bert}, learnable absolute positional embeddings have been applied to the task of masked language modeling~\citep{devlin2018bert,liu2019roberta,Clark2020ELECTRA,Lan2020albert}, Autoregressive-decoding \citep{radford2018improving,radford2019language}, and sequence-to-sequence~\citep{Gehring2017seq2seq,lewis2019bart} settings. Recent work studied ways to extrapolate sinusoidal positional embeddings to longer sequences by randomly shifting absolute positions during training \citep{kiyono2021shiftedAbs} or augmenting with continuous signals \citep{Likhomanenko2021CAPE}.
%While being easy to implement, it is challenging to extend absolute positional embeddings to unseen longer sequence lengths, which is known as the length extrapolation issue~\cite{press2022train}.

\paragraph{Relative.} As opposed to the modeling of absolute position $m$, relative positional embeddings (RPE) that model the positional difference $m-n$ has become popular in the literature~\citep{shaw2018rpe,huang2018music,dai2019transformer,yang2019xlnet,huang2020rpe,he2021deberta,ke2021rethinking,chen2021simple}. In particular, the T5 model that considers bucketed relative distances and log-binning has been shown to perform well on various transformer architectures~\citep{raffel2019exploring}. Rotary positional embedding \citep{su2021roformer} encodes the position with rotations: $f(\bm q_m,m)=R_m\bm q_m$ where $R_m$ is a rotation matrix with angles proportional to $m$. With the rotation's property, the query-key product exhibits a positional difference: $f(\bm q_m,m)^\top f(\bm k_n,n)=\bm q_m^\top R_{n-m}\bm k_n$.

We note that the overview above focuses on the NLP domain. Recent work has applied positional embeddings to other domains such as vision \citep{Kan2021vision} and speech \citep{Likhomanenko2021CAPE}. A survey can be found in \citep{Dufter2022posOverview}. 

\subsection{Kernel and its Application in Transformer}
The kernel trick is a classic approach to generalize the inner product to high dimensional spaces \citep{mika1998kernelsvd,Schplkopf2000cpd,leslie2001kernelsvm,dhillon2004kernelkmeans,takeda2007kernelregress}. In the context of transformers, there has been interest in applying kernels to the self-attention structure to enhance the performance. Examples of such work include kernel for positional embeddings \citep{tsai2019transformer,wu2020transformer,wennberg2021case,luo2021stable}. Another line of research leverages the kernel's feature map \citep{Rahimi2007rndfeature} to linearize the self-attention module and reduce the computational cost \citep{katharopoulos20a,chen2021skyformer,xiong2021nystromformer,peng2021random,choromanski2021performers,qin2022cosformer}. %Although our work does not focus on linearization, we note that a potential benefit of the kernel formulation is to enable linearization.

\section{Theoretical Foundations of CPD Kernels}
\label{sec:cpd}
\subsection{PD and CPD Kernels}
In this work, we use shift-invariant conditionally positive definite (CPD) kernels to model the effect of relative positional differences. We propose this formulation because the notion of \emph{relative} is modeled by a shift-invariant function: a bivariate function $k$ over two positions $(m,n)$ such that $k(m,n)=f(m-n)$ for some univariate $f$. The notion of~\emph{positional difference} $m-n$ is generalized by the CPD kernel. We review the definitions of PD and CPD kernels below.

\begin{definition}[PD Kernel]
A (real) symmetric function $k:\mathcal{X}\times\mathcal{X}\rightarrow\mathbb{R}$ is a positive definite kernel if for any integer $N$ and any $\{x_i\in\mathcal{X}\}_{i=1}^N$,  $\{c_i\in\mathbb{R}\}_{i=1}^N$, the quadratic form is nonnegative: $\sum_{i=1}^N\sum_{j=1}^Nc_ic_jk(x_i,x_j)\geq 0$.
\label{def:kernel}
\end{definition}

\begin{definition}[CPD Kernel]
A (real) symmetric function $\tilde{k}:\mathcal{X}\times\mathcal{X}\rightarrow\mathbb{R}$ is a conditionally positive definite kernel if for any integer $N$ and any $\{x_i\in\mathcal{X}\}_{i=1}^N$, the quadratic form is conditionally nonnegative: $\sum_{i=1}^N\sum_{j=1}^Nc_ic_j\tilde{k}(x_i,x_j)\geq 0$ for $\{c_i\in\mathbb{R}\}_{i=1}^N$ with $\sum_{i=1}^N c_i=0.$
\label{def:cpd_kernel}
\end{definition}

\begin{fact}[\citet{berg1984harmonic} and Prop. 5 of \citet{Schplkopf2000cpd}]
    Let $\tilde{k}:\mathcal{X}\times\mathcal{X}\rightarrow (-\infty,0]$ be a CPD kernel with $\tilde{k}(x,x)=0~\forall x\in \mathcal{X}$. Then, there exists a Hilbert space $\mathcal{H}$ and a mapping $\phi:\mathcal{X}\rightarrow \mathcal{H}$ such that $\norm{\phi(x)-\phi(x')}^2=-\tilde{k}(x,x')$.
    \label{fact:cpd_metric}
\end{fact}
Fact~\ref{fact:cpd_metric} suggests that CPD kernels generalize distance metrics to high dimensional spaces. Since we are interested in positional differences, we examine modeling the distance between positions using CPD kernels.

However, Fact~\ref{fact:cpd_metric} also implies that CPD kernels do not encode inner products as required by self-attention for the computation of pairwise relations. PD kernels represent inner products. To better understand the effect of CPD kernels on self-attention, we need to establish relations between CPD and PD kernels. As noted in \citet{Schplkopf2000cpd}, if one takes any PD kernel and offsets it by a constant, the result is at least a CPD kernel. In the next subsection, we show that the converse is~\emph{nearly} true: if $\tilde{k}$ is CPD, so is $c+\tilde{k}$ for large enough $c\in\mathbb{R}$ (Lemma~\ref{lemma:cpd_shift}). Therefore, we may generate the CPD kernels of interest and transform them into PD kernels if needed.

\subsection{Constructing PD Kernels From CPD Kernels via Constant Shifts}

In this subsection, we review a few properties of CPD kernels and use these to generate a variety of CPD kernels. Then, we present a lemma that transforms CPD kernels into PD kernels via constant shifts. This enables the production of a family of PD kernels from CPD kernels. Finally, we present our critical observation that the exact value of the constant shift is not needed, thanks to a nice property of Softmax normalization.

Below are some important facts about CPD kernels.
\begin{fact}[Scaling and Summation]
    If $\tilde{k}_1$ and $\tilde{k}_2$ are CPD, then so are $a\cdot \tilde{k}_1$ (for $a>0$) and $\tilde{k}_1+\tilde{k}_2$.
    \label{fact:cpd_prop_basic}
\end{fact}
\begin{fact}[\citet{berg1984harmonic} and Prop. 4 of \citet{Schplkopf2000cpd}]
    If $\tilde{k}:\mathcal{X}\times\mathcal{X}\rightarrow (-\infty,0]$ is CPD, then so are $-(-\tilde{k})^\alpha$ for $0<\alpha<1$ and $-\log(1-\tilde{k})$.
    \label{fact:cpt_prop}
\end{fact}
\begin{fact}[Page 3 of \citet{Schplkopf2000cpd}]
    The negative squared distance $-\norm{x-x'}^2$ is CPD.
    \label{fact:negdis_prop}
\end{fact}

% Combining with the fact that the negative squared distance is $-\norm{x-x'}^2$ is CPD (proved in \citet{Schplkopf2000cpd}[p.3]),
The three Facts above jointly yield a rich family of CPD kernels as shown below.

\begin{corollary}
    The following are CPD kernels.
    \begin{enumerate}[topsep=-3pt, itemsep=-3pt,label=(\alph*)]
        \item $\tilde{k}(x,x')=-a\norm{x-x'}^p$ with $0<p\leq 2$ and $a>0$.
        \item $\tilde{k}(x,x')=-b\cdot \log(1+a\norm{x-x'}^p)$ with $0<p\leq 2$ and $a,b>0$.
    \end{enumerate}
    \label{cor:cpd_examples}
\end{corollary}
We note that it is possible to keep iterating between Fact~\ref{fact:cpd_prop_basic} and \ref{fact:cpt_prop} and generate more complicated examples, e.g., $-a\norm{x-x'}^p-b\cdot \log(1+a\norm{x-x'}^p)$ or $-b\cdot \log(1+a\norm{x-x'}^p)^c$ for $0<c<1$. However, since relative positional embeddings are of our interest, we only consider simple CPD kernels. Those with complicated forms are deferred to future work.

Now that Corollary~\ref{cor:cpd_examples} has presented a few class of CPD kernels, we prove a lemma (in Appendix~\ref{appendix:cpd_shift_proof}) that constructs PD kernels from CPD kernels through shifting. Later in Eq.~\eqref{eq:krpe-deriv}, we will see that the shifting construction is combined neatly with the Softmax normalization of self-attention.
\begin{lemma}[CPD Shift Lemma. Proof in Appendix~\ref{appendix:cpd_shift_proof}]
	Let $\tilde{k}:\mathcal{X}\times\mathcal{X}\rightarrow\mathbb{R}$ be a CPD kernel. There exists $c\geq 0$ such that $c+\tilde{k}$ is a PD kernel.
	\label{lemma:cpd_shift}
\end{lemma}
Lemma~\ref{lemma:cpd_shift} implies the CPD kernels in Corollary~\ref{cor:cpd_examples} can be made PD if a large enough constant is added. For example, $c-\norm{x-x'}^p$ for large enough $c$. Although Lemma~\ref{lemma:cpd_shift} does not have an explicit construction of $c$, thanks to the shift-invariant property of the Softmax normalization, we can leave it as an under-determined constant in our positional embedding design (Eq.~\eqref{eq:krpe} in section~\ref{sec:krpe}). Given a set of test points $\{x_i\}_{i=1}^N$, one can do a geometric sequence search\footnote{By geometric sequence search, we can enlarge $c$ by 2, 4, 8, 16, and so on until we find the required large enough constant.} to search for a $c$ such that the $N\times N$ matrix $[c+\tilde{k}(x_i,x_j)]_{i,j=1}^N\succeq 0$. Hence, we do not need the value of $c$ , but we can compute it if needed, e.g., deriving the feature map of $c+\tilde{k}$.

\paragraph{Alternative Proof of $\pmb{c-\norm{x-x'}^p}$.} While the CPD shift lemma is convenient, one can prove $c-\norm{x-x'}^p$ is PD for large enough $c$ using a kernel representation theorem in \citet{schoenberg1938metric}. See Appendix~\ref{appendix:alternative} for details.

\section{Kernelized Relative Positional Embedding}
\label{sec:krpe}
Let $\{\bm q_m\}_{m=1}^L$ and $\{\bm k_n\}_{n=1}^L$ be the input queries and keys. Let $(r_1,...,r_\ell)$ be learnable parameters. We propose a kernelized relative positional embedding as follows.
\begin{equation}
    a_{m,n}=\frac{\exp\big((\bm q_m^\top \bm k_n  +\tilde{k}_{r_1,...,r_\ell}(m,n))/\sqrt{d}\big)}{\sum_{i=1}^L \exp((\bm q_m^\top \bm k_i+\tilde{k}_{r_1,...,r_\ell}(m,i))/\sqrt{d})},
    \label{eq:krpe}
\end{equation}
where $\tilde{k}_{r_1,...,r_\ell}(m,n)$ is any shift-invariant CPD kernel with $\ell$ parameters. Due to Lemma~\ref{lemma:cpd_shift}, Eq.~\eqref{eq:krpe} can be reformulated into its kernel form as follows.

\begin{equation}
\begin{split}
    a_{m,n}\overset{(*)}{=}&\frac{\exp\big((\bm q_m^\top \bm k_n  +c+\tilde{k}_{r_1,...,r_\ell}(m,n))/\sqrt{d}\big)}{\sum_{i=1}^L \exp((\bm q_m^\top \bm k_i+c+\tilde{k}_{r_1,...,r_\ell}(m,i))/\sqrt{d})}\\
    \overset{\text{Lemma}~\ref{lemma:cpd_shift}}{=}&\frac{\exp\big(\bm q_m^\top \bm k_n+k_{r_1,...,r_\ell}(m,n))/\sqrt{d}\big)}{\sum_{i=1}^L \exp(\bm q_m^\top \bm k_i +k_{r_1,...,r_\ell}(m,i))/\sqrt{d})}=\frac{\exp\big(k^{\text{comp}}([\bm q_m,m],[\bm k_n,n])/\sqrt{d}\big)}{\sum_{i=1}^L \exp\big(k^{\text{comp}}([\bm q_m,m],[\bm k_i,i])/\sqrt{d}\big)}.
\end{split}
\label{eq:krpe-deriv}
\end{equation}

(*) is due to the shift-invariant property of the Softmax normalization: $\frac{\exp(x_i)}{\sum_j \exp(x_j)}=\frac{\exp(x_i+c)}{\sum_j \exp(x_j+c)}$ for any $c\in \mathbb{R}$. The second equality defines a~\emph{bias kernel} which is positive definite using Lemma~\ref{lemma:cpd_shift}:
\begin{equation}
    k_{r_1,...,r_\ell}= c+\tilde{k}_{r_1,...,r_\ell}.
    \label{eq:bias_kernel}
\end{equation}
The last equality introduces a~\emph{composite kernel} $k^{\text{comp}}:\mathbb{R}^{d+1}\times \mathbb{R}^{d+1}\rightarrow \mathbb{R}$ as
\begin{equation}
    k^{\text{comp}}([\bm q_m,m],[\bm k_n,n])= \bm q_m^\top \bm k_n + k_{r_1,...,r_\ell}(m,n).
    \label{eq:comp_ker}
\end{equation}

\paragraph{Interpretation.} The proposed method can be interpreted as applying a composite kernel to self-attention. The composite kernel combines the information from query $\bm q_m$, key $\bm k_n$, and positions $(m,n)$ in a way that augments the original self-attention structure by multiplicative and additive position embeddings. The augmentation allows $k^{\text{comp}}$ to not only retain the original $\bm q_m^\top \bm k_n$ but also include positional information from the bias kernel $k_{r_1,...,r_\ell}$.

\paragraph{Practical Choice.} In section~\ref{sec:experi_results}, we fix $\ell=2$ and experiment on two variants of the composite kernel, Eq.~\eqref{eq:comp_ker}, where we call these the~\emph{power} variant and  the~\emph{logarithmic} variant of our proposed KERPLE framework, Eq.~\eqref{eq:krpe-deriv}. These are from a combination of Corollary~\ref{cor:cpd_examples} and Eq.~\eqref{eq:bias_kernel}.

\begin{enumerate}[topsep=-3pt, itemsep=-1pt,label=(\alph*),leftmargin=22mm]
    \item[(power)] $k^{\text{comp}}([\bm q_m,m],[\bm k_n,n])= \bm q_m^\top \bm k_n +c-r_1|m-n|^{r_2}$ with $r_1>0$ and $0<r_2\leq 2$.
    \item[(logarithmic)] $k^{\text{comp}}([\bm q_m,m],[\bm k_n,n])= \bm q_m^\top \bm k_n +c-r_1\cdot \log(1+r_2|m-n|)$ with $r_1,r_2>0$.
\end{enumerate}
We note that these are not the only variants of the composite kernel. In section~\ref{sec:ablation}, we experiment with two more complicated variants, but only find lower training speeds and marginal improvement in perplexities (e.g., logarithmic variant vs. 3-para-log). Thus, based on our study, the choices above hold advantages in both performance and speed. 

\paragraph{Connection to Prior Work.} When the bias kernel, Eq.~\eqref{eq:bias_kernel}, is a triangle kernel: $c-|m-n|$, our model reduces to ALiBi \citep{press2022train}. \citet{wennberg2021case} discuss the situation where the bias kernel is a Gaussian kernel. \citet{tsai2019transformer} is the case where there is no bias kernel and the attention product $\bm q_m^\top \bm k_n$ is multiplied by an exponentiated inner product kernel, $\exp(\bm x^\top \bm y)$. Since ALiBi is the state-of-the-art and has great input length extrapolation, we will focus on comparison with ALiBi in our experiments.

The logarithmic variant has an implicit connection to T5 positional bias \citep{raffel2019exploring}. According to the official GitHub repository \url{https://github.com/google-research/text-to-text-transfer-transformer} and the HuggingFace Transformer \citep{wolf2020huggingface}, T5 bias is implemented with a log-binning strategy. For each head of the transformer, they maintain a bucket of 32 learnable parameters and assign the relative positional bias $b_{m-n}$ to these parameters as
$$
b_{m-n}=\begin{cases}
\text{bucket}[0] & \text{if~}m-n< 0\\
\text{bucket}[m-n] & \text{if~}0\leq m-n<16\\
\text{bucket}[\min(31, \lfloor \frac{\log((m-n)/16)}{\log(128/16)})\cdot 16\rfloor] & \text{if~}m-n\geq 16,
\end{cases}
$$
where $\lfloor\cdot\rfloor$ is the floor function. Note that the log factor is approximately $7.7\log\frac{m-n}{16}$. Therefore, T5 is using a logarithmic bucket assignment, which turns out to extrapolate to different input lengths. Compared with T5, our logarithmic variant uses less parameters (2x12 vs. 32x12) but cannot learn non-monotonic relations (the log function is monotonic). We will conduct more comparisons with T5 bias in our experiments.

\section{Experiments}
\label{sec:experi}

\subsection{Dataset and Implementation Description}
\paragraph{Dataset.} We conduct experiments on OpenWebText2, GitHub, and ArXiv datasets gathered in \citet{gao2020pile}. OpenWebText2 includes recent content from Reddit submissions until 2020, content from multiple languages, document metadata, multiple dataset versions, and open-source replication code. GitHub includes open-source repositories written in primary coding languages such as Java, C/C++, Python, and Go. ArXiv includes papers written in LaTex in Math, Computer Science, Physics, and some related fields. These tasks are motivated by the downstream applications such as online chatting \citep{roller2021recipes}, code completion \citep{chen2021codex}, and academic paper summarization \citep{zhang2020pegasus}.
\begin{table}[!ht]
    \centering
    \caption{\textbf{Dataset Overview.} Raw Size is the size before any up- or down-sampling.}
    \begin{tabular}{lccc}
    \hline\hline
    & OpenWebText2 & GitHub & ArXiv\\ \hline
    Raw Size & 66.77 GB & 95.16 GB & 56.21 GB\\
    Type & Internet & Coding & Academic\\
    \hline\hline
    \end{tabular}
    \label{tab:dataset}
    \vspace{-3mm}
\end{table}

\paragraph{Implementation.} We adapt our model from GPT-NeoX \citep{gpt-neox}, a transformer implementation by the EleutherAI team. The codebase is based on NVIDIA Megatron Language Model \citep{shoeybi2019megatron} and further accelerated using Microsoft DeepSpeed library \citep{rasley2020deepspeed}.

Our model is trained on a machine with one NVIDIA A100 GPU with 40 GB of memory. We adopt almost all configurations of small GPT-NeoX\footnote{\url{https://github.com/EleutherAI/gpt-neox/blob/main/configs/small\_bf16.yml}}, except that we change the train-micro-batch-size to 32, seq-length to 512, and max-position-embeddings to 512. Table~\ref{tab:model_configs} summarizes the important configurations fixed throughout our experiments.
\begin{table}[!ht]
    \centering
    \setlength{\tabcolsep}{3pt}
    \caption{\textbf{162M Model Configurations.}}
    \begin{tabular}{ccccc}
        \hline\hline
         \# Layers & Hidden Size & \# Attention Heads & Train Seq. Len. & \# Trainable Params.\\
         12 & 64 & 12 & 512 & ~162M\\ \hline
         Optimizer & Batch Size & Train Steps & Precision & \# Trainable Params. for RPEs\\
         Adam (lr 6e-4) & 32 & 50,000 & bfloat16 & at most 36\\
         \hline\hline
    \end{tabular}
    \label{tab:model_configs}
\end{table}
In particular, the floating-point encoding is set as bfloat16 (Brain Floating Point, developed by Google Brain) so that the training can be accelerated by half-precision computation with reliable stability \citep{kalamkar2019bf16}. Hidden size 64 means that $d=64$ in Eq.~\eqref{eq:krpe}.

\subsection{Experimental Results (Also c.f. Appendix~\ref{appendix:more-experi} to~\ref{sec:hyper_model})}
\label{sec:experi_results}
We conduct experiments to cover aspects such as input length extrapolation, application on different domains, and comparison with the prior work. These are elaborated on below. (i) Motivated by the input length extrapolation demonstrated in \citep{press2022train}, we train our model with length 512 and test on lengths ranging from 512 to 16384. We hope that the emphasis on extrapolation enables the application of transformers to longer sequences. (ii) To evaluate the applicability of the model in different domains, we conduct experiments on OpenWebText2, GitHub, and ArXiv datasets. (iii) To validate the effectiveness of our method, we compare KERPLE with Sinusoidal \citep{vaswani2017attention}, Rotary \citep{su2021roformer}, T5 \citep{raffel2019exploring}, and ALiBi \citep{press2022train}.

Table~\ref{tab:openweb-github-arxiv} reports the perplexities at different extrapolation lengths. We perform non-overlapping evaluation: Suppose text is segmented in a different manner for 512 and 1024 tokens, we have N sentences and N/2 correspondingly to evaluate. We also perform a paired two-sided t-test to validate the statistical significance (significance level=0.05). We compare each candidate RPE with our proposed logarithmic variant and mark the candidate with a $^\dagger$~\emph{if the log variant is statistically significantly better}. Table~\ref{tab:speed} reports the training speeds. These tables yield three conclusions. First, within the KERPLE framework, the logarithmic variant is better than the power variant. Secondly, the logarithmic variant is 9.7\% faster than T5. In terms of extrapolation, the logarithmic variant generally does better than T5 but could be slightly worse than T5 at shorter lengths. Third, the logarithmic variant is slightly slower than some prior work (ALiBi, Rotary, and Sinusoidal) but consistently outperform these methods at all extrapolation lengths. More details are given below.

\paragraph{Logarithmic Variant vs. Power Variant.} In our proposed KERPLE framework, the logarithmic variant is better than the power variant. Precisely, the logarithmic variant is 4.4\% faster and has lower perplexities across all extrapolation lengths and all tasks.

\paragraph{Logarithmic Variant vs. T5.} In terms of speed, the logarithmic variant is 9.7\% faster than T5. In terms of extrapolation perplexity, the logarithmic variant is close to or slightly worse than T5 when the extrapolation length is shorter than 2048, and consistently excels T5 at longer extrapolation lengths. The tendency of extrapolation holds for all datasets evaluated in this work. 

\paragraph{Logarithmic Variant vs. ALiBi, Rotary, and Sinusoidal.} The logarithmic variant is 1.6\% slower, 7.5\% faster, and 3.0\% slower than ALiBi, Rotary, and Sinusoidal. The speed comparison makes sense because we require only a limited amount of learnable parameters for RPEs (at most $3\cdot H$). Also, the logarithmic variant consistently outperforms prior work at all extrapolation lengths and tasks.

\begin{table}[!ht]
    \setlength{\tabcolsep}{2pt}
    \centering
    \caption{\textbf{Perplexity Comparison on OpenWebText2, GitHub, and ArXiv.} All models are trained for 50k steps with training length 512 and five random seeds. $x^\dagger$ means our log variant is statistically significantly~\emph{better} than $x$. The test used is paired two-sided t-test with $\alpha=0.05$.}
    \begin{tabular}{@{\extracolsep{3pt}}lcccccc}
    \hline\hline
    \multicolumn{7}{c}{\textbf{OpenWebText2}}\\
    \hline
    \multirow{2}{*}{Extrp.} & \multicolumn{2}{c}{KERPLE} &
     \multirow{2}{*}{ALiBi} & \multirow{2}{*}{T5} & \multirow{2}{*}{Rotary} & \multirow{2}{*}{Sinusoidal}\\
     \cline{2-3}
     &   (log) & (power) &  &  &  & \\ \hline
    512 & 23.9 $\pm$ 0.6 & 23.9 $\pm$ 0.6 & 23.9 $\pm$ 0.6 & \textbf{23.7 $\pmb{\pm}$ 0.6} & 24.2 $\pm$ 0.6$^\dagger$ & 33 $\pm$ 1$^\dagger$\\
1024 & 22.0 $\pm$ 0.6 & 22.1 $\pm$ 0.7 & 22.4 $\pm$ 0.5$^\dagger$ & \textbf{21.9 $\pmb{\pm}$ 0.6} & 32.8 $\pm$ 1.7$^\dagger$ & 750 $\pm$ 346$^\dagger$\\
2048 & \textbf{21.6 $\pmb{\pm}$ 0.3} & 21.9 $\pm$ 0.2$^\dagger$ & 22.5 $\pm$ 0.2$^\dagger$ & 21.7 $\pm$ 0.2 & 62.4 $\pm$ 6.1$^\dagger$ & 5507 $\pm$ 2607$^\dagger$\\
4096 & \textbf{21.2 $\pmb{\pm}$ 0.4} & 21.5 $\pm$ 0.5$^\dagger$ & 22.2 $\pm$ 0.4$^\dagger$ & 22.5 $\pm$ 0.6$^\dagger$ & 111 $\pm$ 13.8$^\dagger$ & 14039 $\pm$ 2325$^\dagger$\\
8192 & \textbf{21.3 $\pmb{\pm}$ 0.4} & 21.6 $\pm$ 0.4$^\dagger$ & 22.3 $\pm$ 0.3$^\dagger$ & 25.5 $\pm$ 1.3$^\dagger$ & 185 $\pm$ 18.9$^\dagger$ & 22621 $\pm$ 1927$^\dagger$\\
16384 & \textbf{21.4 $\pmb{\pm}$ 0.6} & 21.6 $\pm$ 0.6 & 22.5 $\pm$ 0.5$^\dagger$ & 31.4 $\pm$ 3.1$^\dagger$ & 269 $\pm$ 33.0$^\dagger$ & 30046 $\pm$ 4824$^\dagger$\\
    \hline\hline
    \multicolumn{7}{c}{\textbf{GitHub}}\\
    \hline
    \multirow{2}{*}{Extrp.}& \multicolumn{2}{c}{KERPLE} & \multirow{2}{*}{ALiBi} & \multirow{2}{*}{T5} & \multirow{2}{*}{Rotary} & \multirow{2}{*}{Sinusoidal} \\
     \cline{2-3}
     &   (log) & (power) &  &  &  & \\ \hline
    512 & 3.40 $\pm$ 0.20 & 3.42 $\pm$ 0.20 & 3.42 $\pm$ 0.21 & \textbf{3.38 $\pmb{\pm}$ 0.21} & 3.44 $\pm$ 0.20$^\dagger$ & 4 $\pm$ 0.2$^\dagger$\\
1024 & 3.04 $\pm$ 0.14 & 3.07 $\pm$ 0.16 & 3.15 $\pm$ 0.17$^\dagger$ & \textbf{3.02 $\pmb{\pm}$ 0.14} & 3.86 $\pm$ 0.25$^\dagger$ & 105 $\pm$ 39$^\dagger$\\
2048 & 2.86 $\pm$ 0.10 & 2.90 $\pm$ 0.08$^\dagger$ & 3.13 $\pm$ 0.10$^\dagger$ & \textbf{2.84 $\pmb{\pm}$ 0.09} & 5.94 $\pm$ 0.64$^\dagger$ & 1380 $\pm$ 404$^\dagger$\\
4096 & \textbf{2.74 $\pmb{\pm}$ 0.05} & 2.79 $\pm$ 0.06 & 3.04 $\pm$ 0.08$^\dagger$ & 2.78 $\pm$ 0.04$^\dagger$ & 11.1 $\pm$ 1.55$^\dagger$ & 5217 $\pm$ 1118$^\dagger$\\
8192 & \textbf{2.71 $\pmb{\pm}$ 0.05} & 2.76 $\pm$ 0.05 & 3.04 $\pm$ 0.03$^\dagger$ & 2.95 $\pm$ 0.13$^\dagger$ & 20.2 $\pm$ 2.75$^\dagger$ & 10081 $\pm$ 3583$^\dagger$\\
16384 & \textbf{2.75 $\pmb{\pm}$ 0.16} & 2.76 $\pm$ 0.13 & 3.02 $\pm$ 0.13$^\dagger$ & 3.35 $\pm$ 0.27$^\dagger$ & 31.3 $\pm$ 5.20$^\dagger$ & 16443 $\pm$ 8503$^\dagger$\\
    \hline\hline
    \multicolumn{7}{c}{\textbf{ArXiv}}\\
    \hline
    \multirow{2}{*}{Extrp.} & \multicolumn{2}{c}{KERPLE} & \multirow{2}{*}{ALiBi} & \multirow{2}{*}{T5} & \multirow{2}{*}{Rotary} & \multirow{2}{*}{Sinusoidal}\\
      \cline{2-3}
     &   (log) & (power) &  &  &  & \\ \hline
    512 & 6.07 $\pm$ 0.26 & 6.10 $\pm$ 0.26 & 6.12 $\pm$ 0.26$^\dagger$ & \textbf{6.03 $\pmb{\pm}$ 0.26} & 6.07 $\pm$ 0.27 & 43 $\pm$ 44\\
1024 & 5.61 $\pm$ 0.10 & 5.65 $\pm$ 0.10$^\dagger$ & 5.82 $\pm$ 0.09$^\dagger$ & \textbf{5.58 $\pmb{\pm}$ 0.09} & 7.49 $\pm$ 0.34$^\dagger$ & 221 $\pm$ 136$^\dagger$\\
2048 & 5.22 $\pm$ 0.12 & 5.26 $\pm$ 0.13$^\dagger$ & 5.71 $\pm$ 0.14$^\dagger$ & \textbf{5.21 $\pmb{\pm}$ 0.14} & 14.2 $\pm$ 1.81$^\dagger$ & 730 $\pm$ 343$^\dagger$\\
4096 & \textbf{5.20 $\pmb{\pm}$ 0.10} & 5.25 $\pm$ 0.09 & 5.87 $\pm$ 0.08$^\dagger$ & 5.32 $\pm$ 0.16$^\dagger$ & 30.1 $\pm$ 4.32$^\dagger$ & 1998 $\pm$ 497$^\dagger$\\
8192 & \textbf{5.01 $\pmb{\pm}$ 0.10} & 5.06 $\pm$ 0.15 & 5.74 $\pm$ 0.13$^\dagger$ & 5.54 $\pm$ 0.39$^\dagger$ & 54.3 $\pm$ 6.22$^\dagger$ & 4228 $\pm$ 2645$^\dagger$\\
16384 & \textbf{5.07 $\pmb{\pm}$ 0.16} & 5.07 $\pm$ 0.19 & 5.78 $\pm$ 0.15$^\dagger$ & 6.25 $\pm$ 0.61$^\dagger$ & 85.4 $\pm$ 7.40$^\dagger$ & 6674 $\pm$ 5696\\
    \hline\hline
    \end{tabular}
    \label{tab:openweb-github-arxiv}
\end{table}

\begin{table}[!ht]
    \centering
    \caption{\textbf{Training Time Comparison on GitHub}}
    \begin{tabular}{lcccccc}
    \hline\hline
     & \multicolumn{2}{c}{KERPLE} & \multirow{2}{*}{ALiBi} & \multirow{2}{*}{T5} & \multirow{2}{*}{Rotary} & \multirow{2}{*}{Sinusoidal}\\
      \cline{2-3}
     &   (log) & (power) &  &  &  & \\ \hline
    sec/step & 0.307 & 0.321 & 0.302 & 0.340 & 0.332 & 0.298 \\ 
    %samples/sec &  104 & 99.6 & 106 & 94.1 & 96.5 & 107 \\
    \hline\hline
    \end{tabular}
    \label{tab:speed}
\end{table}

\subsection{Experiments on Complicated Kernels}
\label{sec:ablation}
In addition to the practical variants (power \& logarithmic) in section~\ref{sec:krpe}, we consider two complicated versions of the composite kernel, Eq.~\eqref{eq:comp_ker}, as follows.
\begin{enumerate}[topsep=-3pt, itemsep=-1pt, leftmargin=20mm]
    \item[(bias+wht)] bias + weight:\\
    $k^{\text{comp}}([\bm q_m,m],[\bm k_n,n])= \bm q_m^\top \bm k_n\cdot \exp(-r_3|m-n|^{r_4}) + c-r_1|m-n|^{r_2}$\\ 
    with $r_1,r_3>0$ and $0<r_2,r_4\leq 2$.
    \item[(3-para-log)] 3-parameter-logarithmic:\\
    $k^{\text{comp}}([\bm q_m,m],[\bm k_n,n])= \bm q_m^\top \bm k_n+c-r_1\cdot \log(1+r_2|m-n|^{r_3})$\\
    with $r_1,r_2>0$ and $0<r_3\leq 2$.
\end{enumerate}
Recall the tensor product property of a kernel: if $k_1$ is a kernel on $\mathcal{X}$ and $k_2$ is a kernel on $\mathcal{Y}$, then $k((x,y),(x',y'))=k_1(x,x')k_2(y,y')$ is a kernel on $\mathcal{X}\times \mathcal{Y}$. Therefore, (bias+wht) is the setting where we train a weight $\exp(-r_3|m-n|^{r_4})$ and a bias kernel $c-r_1|m-n|^{r_2}$. $\bm q_m^\top \bm k_n$ is multiplied by the weight kernel and then added with the bias kernel. (3-para-log) is the setting where we consider $|m-n|^{r_3}$ in the log. When $r_3=1$, it is reduced to the logarithmic variant proposed in section~\ref{sec:krpe}.

We plug in these composite kernel $k^{\text{comp}}$ into our KERPLE framework, Eq.~\eqref{eq:krpe-deriv}, and test the performance of these RPE. Compared with section~\ref{sec:experi_results}, Table~\ref{tab:add_experi} suggests that these variants do not have clear advantage in extrapolation performance, e.g., 3-para-log is slightly better in perplexity than the (two-parameter) logarithmic variant. Thus, enlarging the complexity of kernels does not necessarily give better performance in the context of RPE.

\begin{table}[!ht]
    \setlength{\tabcolsep}{4pt}
    \centering
    \caption{\textbf{Perplexity Comparison for KERPLE with Complicated Kernels on OpenWebText2, GitHub, and ArXiv.} All models are trained for 50k steps with training length 512 and five seeds random. OOM means out of memory.}
    \begin{tabular}{@{\extracolsep{3pt}}lcccccc}
    \hline\hline
      \multirow{2}{*}{Extrp.} & \multicolumn{2}{c}{OpenWebText2} &\multicolumn{2}{c}{GitHub} &\multicolumn{2}{c}{ArXiv}\\
      \cline{2-3} \cline{4-5} \cline{6-7}
     &  (bias+wht) & (3-para-log) &  (bias+wht) & (3-para-log) & (bias+wht) & (3-para-log) \\ \hline
    512 & 24.1 $\pm$ 0.6 & 23.8 $\pm$ 0.6  & 3.44 $\pm$ 0.21 & 3.40 $\pm$ 0.20 & 6.11 $\pm$ 0.27 & 6.06 $\pm$ 0.27\\
1024 & 22.2 $\pm$ 0.6 & 22.0 $\pm$ 0.7 & 3.08 $\pm$ 0.15 & 3.04 $\pm$ 0.13 & 5.66 $\pm$ 0.09 & 5.61 $\pm$ 0.10\\
2048 & 21.9 $\pm$ 0.4 & 21.6 $\pm$ 0.2 & 2.90 $\pm$ 0.12 & 2.85 $\pm$ 0.10 & 5.28 $\pm$ 0.12 & 5.21 $\pm$ 0.12\\
4096 & 21.5 $\pm$ 0.5 & 21.2 $\pm$ 0.4 & 2.79 $\pm$ 0.06 & 2.73 $\pm$ 0.05 & 5.31 $\pm$ 0.08 & 5.18 $\pm$ 0.09\\
8192 & 21.4 $\pm$ 0.5 & 21.3 $\pm$ 0.4 & 2.76 $\pm$ 0.03 & 2.68 $\pm$ 0.04 & 5.16 $\pm$ 0.18 & 5.00 $\pm$ 0.11\\
16384 & OOM & OOM & OOM & OOM & OOM & OOM\\
    \hline\hline
    \end{tabular}
    \label{tab:add_experi}
\end{table}

%\begin{table}[!ht]
    %\setlength{\abovecaptionskip}{2pt}
%    \centering
%    \caption{\textbf{Training Time of KERPLE with Complicated Kernels on GitHub Dataset.} }
%    \begin{tabular}{lcc}
%    \hline\hline
     %& \multicolumn{2}{c}{KERPLE} \\
     % \cline{2-3}
%     &   bias+wht & 3-para-log \\ \hline
%    sec/step & 0.333 & 0.337\\
%    \hline\hline
%    \end{tabular}
%    \label{tab:add_speed}
%\end{table}

\subsection{Plots of Kernel Functions}
We plot kernel functions including the power, log variants, and ALiBi for different heads to see their contributions to softmax. We use the GitHub dataset for demonstration. Please see Figure~\ref{fig:log_rebuttal},~\ref{fig:power_rebuttal}, and~\ref{fig:alibi_rebuttal}. Both ALiBi and its generalized power variant quickly reach a very negative value. In contrast, the log variant successfully discovers several flat kernels, effectively extending the window attention. This corroborates our previous observation that KERPLE-log can utilize more distant token information.

\begin{figure*}[!ht]
\centering
\caption{Kernel Functions of Learned by the Log Variant.}
\includegraphics[width=0.7\textwidth]{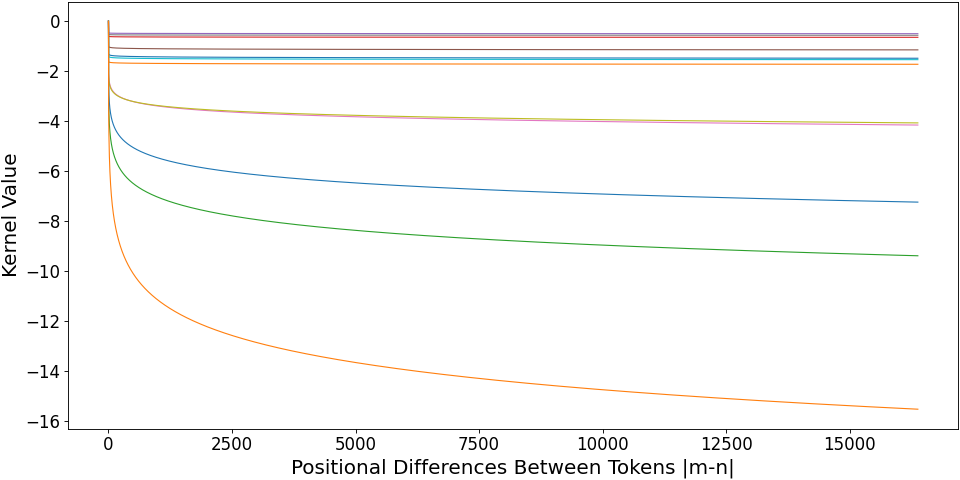}
\label{fig:log_rebuttal}
\end{figure*}

\begin{figure*}[!ht]
\centering
\caption{Kernel Functions Learned by the Power Variant. Note the y-axis should be multiplied by $1e8$, which is a very negative value.}
\includegraphics[width=0.7\textwidth]{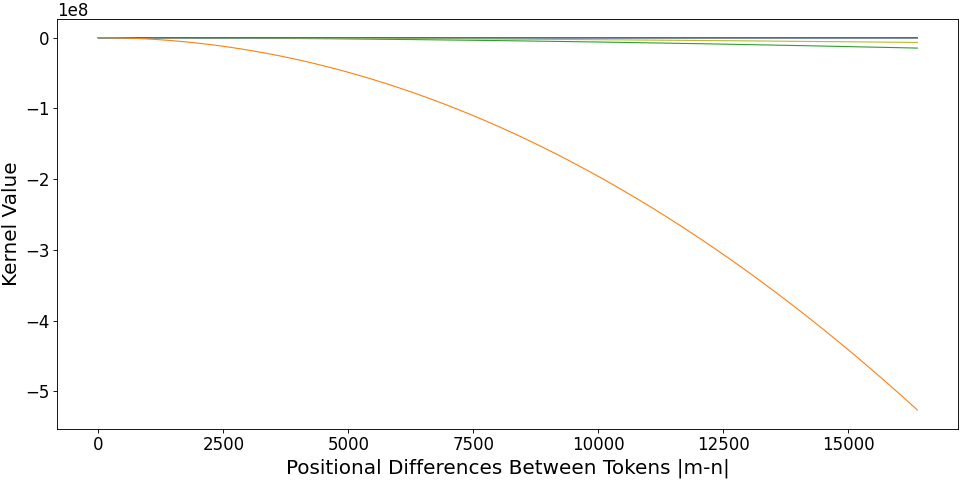}
\label{fig:power_rebuttal}
\end{figure*}

\begin{figure*}[!ht]
\centering
\caption{Kernel Functions Learned by ALiBi.}
\includegraphics[width=0.7\textwidth]{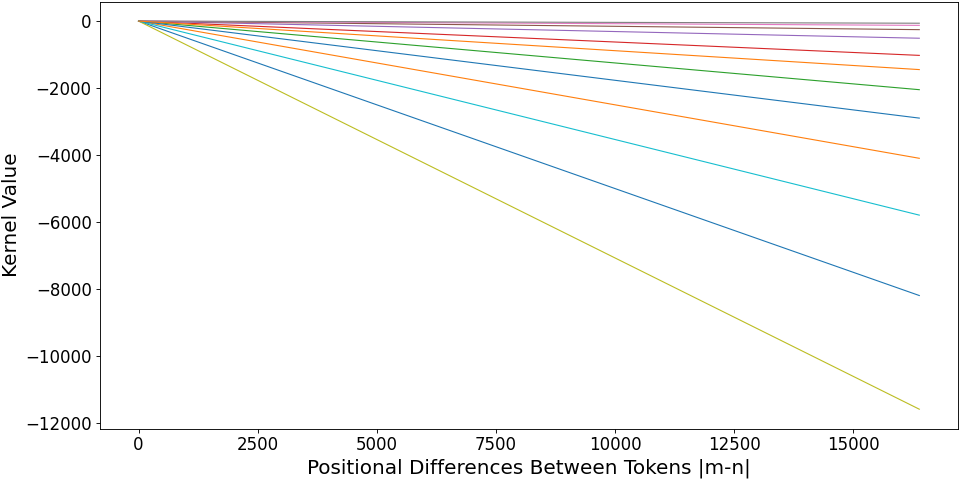}
\label{fig:alibi_rebuttal}
\end{figure*}

\subsection{Position-wise Perplexity Evaluation}
We plot the position-wise perplexity with evaluation length=4096 in Figure~\ref{fig:pos-wise-ppl-window-4096}. Please see Appendix~\ref{sec:posppl_16384} for similar length=16384 result. The evaluation is done by measuring the loss at each position in each sequence and averaging over the sequences.

% First, we notice that as a general trend, tokens that appear at the beginning of an evaluation sequence have much higher PPL; This is expected as these tokens have access to only a limited amount of context (early token curse). Second, we focus on Figure~\ref{fig:pos-wise-ppl}(a) and observe that all methods except rotary can extrapolate smoothly to length=4096. We also observe that our proposed KERPLE-log variant achieves the lowest PPL with a continuously decreasing trend toward length=4096. We want to highlight that the decreasing trend indicates successful extrapolation as our model uses distant token information not seen during training time (training length=512). Finally, we push the evaluation sequence length to 16384 and plot the result in Figure~\ref{fig:pos-wise-ppl}(b); This time we can see that the PPL of T5 bias starts to explode after length=5000, which also corroborates the finding (that T5 bias cannot extrapolate well) in the original ALiBi paper. In addition, we can see that the decreasing trend of our PPL is not as prominent as Figure~\ref{fig:pos-wise-ppl}(a). This is also expected as there is very likely an upper limit of how distant two relevant tokens could be. In summary, using a larger context and embedding is indeed helpful to our model up to length=4096 on the GitHub dataset.

We note that PPL@512 of KERPLE-log is the lowest among all model variants. We can derive several critical observations for evaluation length=4096 in Figure~\ref{fig:pos-wise-ppl-window-4096}: First, KERPLE-log lies below KERPLE-log-windowed@512, indicating its usage of more distant information than window attention: If our model does not use more information other than a fixed-window=512, the y-values after position=512 should overlap with the line windowed at 512. This is clearly not the case. In addition, the PPL of KERPLE-log continues to decrease till the end of 4096 positions (Not plateauing). Second, T5 lies below KERPLE-log-windowed@512 most of the time and fluctuates around KERPLE-log-windowed@512 after length=3000. It is still worse than KERPLE-log. Third, ALiBi lies above KERPLE-log-windowed@512 for almost all the positions, indicating that window attention might be a better choice than ALiBi.

Although window attention is a strong baseline, our KERPLE-log is almost like a free lunch compared to window attention: With only 24 additional learnable parameters (2 para. for each head), the almost same training speed, and the same train length=512 as window attention, it is able to achieve lower PPLs across different positions.

% \begin{figure}[!ht]
%      \centering   
%      \caption{\textbf{Position-wise Perplexity on GitHub at different Evaluation Lengths.}}
%      \subfloat[][\centering Evaluation Length=4096]{{\includegraphics[width=0.48\textwidth]{pos_perplexity_4096_older.pdf} }}%
%      \subfloat[][\centering Evaluation Length=16384]{{\includegraphics[width=0.48\textwidth]{pos_perplexity_16384_older.pdf} }}%
%      \label{fig:pos-wise-ppl}
% \end{figure}

\begin{figure}[!ht]
    \centering   
    \caption{\textbf{Position-wise Perplexity on GitHub at Evaluation Length=4096 Compared to Window Attention@512.}}
    {{\includegraphics[width=\textwidth]{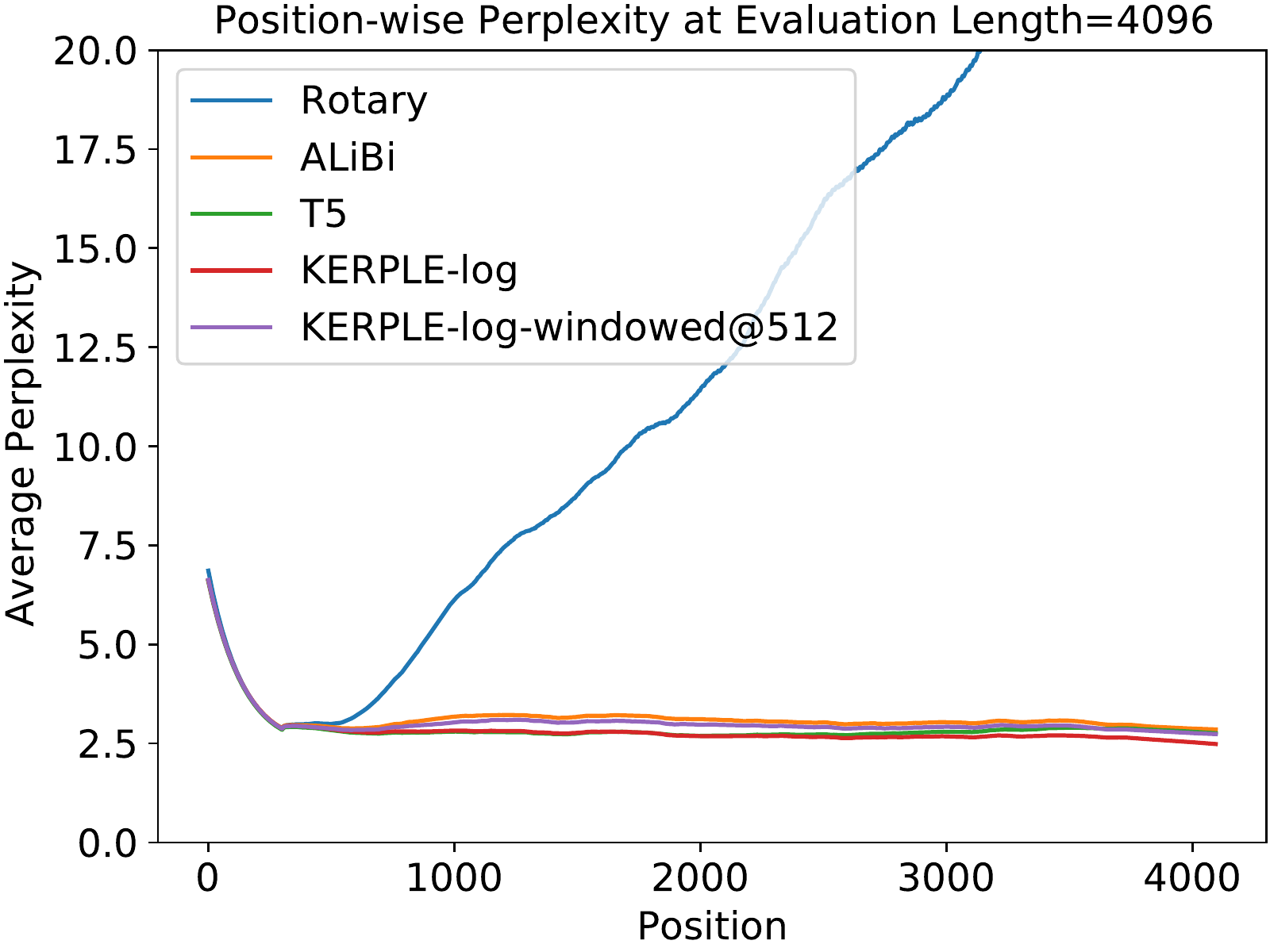} }}%
    \label{fig:pos-wise-ppl-window-4096}
\end{figure}

\section{Conclusion and Future Work}
\label{sec:conclusion}
A general framework, KERPLE, is proposed to kernelize relative positional embeddings for length extrapolation. At the core of this framework is the application of CPD kernels and the derivation of practical variants. We show that these CPD kernels can be implicitly converted to PD kernels, which keep the inner product interpretation of self-attention. We also demonstrate that the logarithmic variant achieves exceptional extrapolation performance on three large language modeling datasets. We believe our work paves the way for some interesting future directions that resolve our limitations. For instance, we can consider general kernel families and model non-monotonic effects due to positional differences. In addition, the use of learnable parameters in KERPLE might enable better generalization to inputs higher than one-dimensional. Last but not least, there is always room for improving memory efficiency by adjusting the model architecture and training procedure.

\section{Broader Impact}
\label{sec:broader_impact}
Our work develops a better understanding of relative positional embedding for transformers based on expressive kernel classes that adapt well to various datasets. The results apply to domains where the positional information is helpful in the modeling, e.g., natural language, programming language, and DNA/protein sequences for biology/medicine. The studies of transformers may have positive economic effects by enabling new tasks which cannot be done by humans or enhancing accuracy and efficiency. But inappropriate use can have negative societal impacts. These include job loss due to automation, the ethical challenges from improper text generation, and the privacy issues in the data collection process. These implications apply to any research on natural language processing and are not associated with any specific work.

\section{Acknowledgement}
We thank the anonymous reviewers for their insightful feedback and suggestions. We thank Princeton Research Computing for the technical support on the Della and the Adroit clusters. The third author acknowledges support from NSF MRI Award: 1919452.

\bibliographystyle{unsrtnat}
\bibliography{mybib}

\appendix
\setcounter{lemma}{0}

\section{Appendix}

\subsection{Proof of CPD Shift Lemma}
\label{appendix:cpd_shift_proof}
\begin{lemma}[CPD Shift Lemma]
	Let $k:\mathcal{X}\times\mathcal{X}\rightarrow\mathbb{R}$ be a conditionally positive definite (CPD) kernel. Then, there exists $c\geq 0$ such that $c+k(x,y)$ is a positive definite kernel.
\end{lemma}

\begin{proof}
Let $K=[k(x_i,x_j)]_{i,j=1}^N$ be the matrix generated by $\{x_1,...,x_N\}$ with $N\in\mathbb{N}$. Consider 
$$f_c(v)= v^\top( c\one\one^\top +K)v=c(v^\top \one)^2 + v^\top K v.$$ We want to show there exists a large enough $c$ such that $f_c(v)\geq 0$ for all $v\in \{v:\norm{v}=1\}$.

\begin{enumerate}[label=\textbf{(\roman*)},leftmargin=15pt]
    \item \textbf{It is sufficient to consider $\pmb{a^*=\min_{v:\norm{v}=1}v^\top Kv<0.}$} 
    
    Let $a^*$ be the solution to the minimization:
    $$a^*=\min_{v:\norm{v}=1}v^\top Kv.$$
	Since $v^\top Kv$ is continuous in $v$ and $\{v:\norm{v}=1\}$ is compact (i.e., closed and bounded), $a^*$ must exist. If $a^*\geq0$, $K$ is positive semidefinite and $f_c(v)\geq 0$ for $c\geq 0$. Thus, without loss of generality, we assume $a^*< 0$.
	\item \textbf{It is sufficient to consider $\pmb{K}$ without zero eigenvalues (i.e., full rank).}
	
	If there exists $v_0$ such that $Kv_0=0$, then $c\geq 0 $ is enough to satisfy $f_c(v_0)\geq 0$. For any $v_1$ satisfying $v_1^\top v_0=0$, we have $(v_1+v_0)^\top K (v_1+v_0)=v_1^\top K v_1$. Therefore, whether there exists $c$ to have $f_c(v)\geq 0$ doesn't depend on the eigenvector corresponding to zero eigenvalue (if there is such a vector). This means it is enough to consider $K$ without zero eigenvalues.
	\item \textbf{It is sufficient to consider strict CPD.}
	
	By definition of conditional positive definiteness (CPD), we know $v^\top K v\geq0$ when $v^\top \one=0$. Since $K$ has no zero eigenvalue, we cannot have $v^\top K v=0$ when $v^\top \one=0$\footnote{By spectral decomposition, $v^\top K v = \sum_i\lambda_i (v^\top u_i)^2\geq 0$. Since there is no $\lambda_i=0$, the inequality is strict.}. This means the inequality is strict here: $v^\top K v>0$ when $v^\top \one=0$, and it is enough to consider strict CPD.
	
	\item \textbf{It is sufficient to show there exists (small enough) $\pmb{\delta>0}$ such that} $\pmb{v'\in T_\delta\Rightarrow v'^\top K v'>0}$.
	
	Note $T_\delta$ is defined as $$T_\delta=\{v':|v'^\top \one|<\delta,~\norm{v'}=1\}.$$
	For any $v'\in T_\delta$, if $v'$ satisfies $v'^\top K v'>0$, then $c\geq 0$ is enough to have $f_c(v')\geq 0$. Conversely, when $|v^\top \one|\geq\delta$ and $\norm{v}=1$, observe that
	$$f_c(v)=c(v^\top \one)^2 + v^\top K v\geq c(v^\top \one)^2 + a^*\geq c\delta^2 +a^*$$
	Then, $f_c(v)\geq 0$ when $c\geq\frac{-a^*}{\delta^2}$. Therefore, we need to prove $v'\in T_\delta\Rightarrow v'^\top K v'>0$ for small enough $\delta$.
	\item \textbf{Leveraging Continuity.} Consider $v'$ satisfying $\norm{v'-v}<\delta_2$ with $v\in S=\{v:\norm{v}=1,~v^\top \one=0\}$. Since $v^\top Kv$ is continuous in $v$ and $\norm{K}<\infty$, for any $\epsilon>0$ and any $v\in S$, a small enough $\delta_2>0$ gives $|v'^\top Kv'-v^\top Kv|<\epsilon$. 
	
	To see this, taking $v'=v+p$ with $\norm{p}<\delta_2$, we have
	$$|v'^\top Kv'-v^\top Kv|=|p^\top Kv + v^\top Kp+p^\top Kp|\underset{\norm{v}=1}{\leq}\norm{K}(2\norm{p}+\norm{p}^2)\leq \norm{K}(2\delta_2 +\delta_2^2).$$
	Therefore, $0<\delta_2<\sqrt{1+\epsilon/\norm{K}}-1$ is enough to have $|v'^\top Kv'-v^\top Kv|<\epsilon$. 
	
	By definition of strict CPD, we know $\min_{v\in S} v^\top Kv=\lambda>0$. Thus, take $\epsilon<\lambda$, a small enough $\delta_2$ gives $v'^\top Kv'>v^\top K v-\epsilon>=\lambda -\epsilon>0$. 
	In other words, there exists a small enough $\delta_2$ such that $v'^\top Kv'>0$ for $v'\in S_{\delta_2} =\{v':\norm{v'-v}<\delta_2,~v\in S\}$.
	\item \textbf{Proving $\pmb{\exists~\delta>0~\text{s.t.}~v'\in T_\delta\Rightarrow v'^\top K v'>0}$}.
	
	Due to (iv), we want to show $\exists~\delta>0~\text{s.t.}~v'\in T_\delta\Rightarrow v'^\top K v'>0$. We will prove by the conclusion of (v). Let $\norm{v'}=1$,~$v'^\top \one =r$ with $|r|<\delta$ and $v''=v'-\frac{r}{n}\one$. We have
	$$v''^\top \one=0,~~~\norm{v'-v''}=\frac{r}{\sqrt{n}},~~~\norm{v''}=\sqrt{\norm{v'-\frac{r}{n}\one}^2}=\sqrt{1-\frac{r^2}{n}}.$$
	
	Take $v=\frac{v''}{\norm{v''}}=\frac{v''}{q}$, where $q=\sqrt{1-\frac{r^2}{n}}$. We have $v^\top \one=0,~~~\norm{v}=1$ and
	\begin{align*}
        \norm{v'-v}^2=&\norm{(1-\frac{1}{q})v' + \frac{1}{q}\frac{r}{n}\one}^2=(1-\frac{1}{q})^2 + \frac{r^2}{nq^2}+2(1-\frac{1}{q})\frac{1}{q}\frac{r^2}{n}\\
        =&(1-\frac{1}{q})^2+\frac{r^2}{n}(\frac{2}{q}-\frac{1}{q^2})=\frac{1}{q^2}\Big( (q-1)^2 + \frac{r^2}{n}(2q-1) \Big)\\
        =&\frac{1}{1-\frac{r^2}{n}}\Big( 1-\frac{r^2}{n}-2q+1+ \frac{r^2}{n}2q - \frac{r^2}{n} \Big)\\
        =&\frac{1}{1-\frac{r^2}{n}}\Big( 1-\frac{r^2}{n}-2q(1-\frac{r^2}{n})+1-\frac{r^2}{n} \Big) = 2-2q\\
        =&2(1-\sqrt{1-\frac{r^2}{n}}) \approx 2(1-1+\frac{1}{2}\frac{r^2}{n})=\frac{r^2}{n}~~~~(\sqrt{1-x}\approx 1-\frac{x}{2}~\text{when~}|x|\ll 1)
    \end{align*}
	Thus, $\norm{v'-v}=O(\frac{|r|}{\sqrt{n}})\leq O(\frac{\delta}{\sqrt{n}})<\delta_2$ for small enough $\delta$. This implies that, with a small enough $\delta$, for any $v'\in T_\delta$, we can find $v\in S$ such that $\norm{v'-v}<\delta_2$. Thus, $v' \in S_{\delta_2}$, and by (v), we arrive at $v'^\top Kv'>0$.
\end{enumerate}
\end{proof}

\subsection{Shift-invariant Kernels with Bounded and Unbounded Ranges}
\label{appendix:alternative}

Definition~\ref{def:kernel} implies a shift-invariant kernel is generated by a univariate function $f:\mathcal{X}\rightarrow \mathbb{R}$. To characterize the set of valid univariate functions, we introduce the positive definite functions as below.
\begin{definition}[Positive definite function]
A (real) positive definite function is a function $f:\mathcal{X}\rightarrow\mathbb{R}$ such that for any integer $N$ and any set of $N$ points $\{x_i\in\mathcal{X}\}_{i=1}^N$, the $N\times N$ matrix $\bm A=[f(x_i-x_j)]_{i,j=1}^N$ is positive semidefinite.
\label{def:pd}
\end{definition}

We will interchange the ideas of shift-invariant kernels and positive definite functions because they are equivalent by definition. Any statement in positive definite functions can be translated into shift-invariant kernels, and vice versa. Because of this, we will use some facts about the positive definite functions to derive the shift-invariant kernels of our interest.

\paragraph{Generalizing Classical Bounded Shift-invariant Kernel.} In the literature, there have been studies on applying kernels in the attention mechanism \citep{tsai2019transformer,choromanski2021performers,peng2021random}. One of the most common approaches is to consider the Gaussian kernel:
$$k(m,n)=\exp(-\gamma (m-n)^2),~~~\gamma>0.$$
Note the Gaussian kernel is bounded ($k(m,n)\in (0,1]$ for the case above). To generalize it to a broader class of bounded shift-invariant kernels, observe that the Gaussian kernel generated by a positive definite function of the form $f(x)=\exp(-|x|^2)$. Since there is no strong reason to stick to the power of 2, one may generalize it to a broader class of positive definite functions as below.

\begin{fact}[Corollary 3 of \citet{schoenberg1938metric}]
$\exp (-|x|^p)$ is positive definite if $0<p\leq 2$ and not positive definite if $p>2$.
\label{fact:exp}
\end{fact}

Fact~\ref{fact:exp} implies that, if one wants to find a class of bounded shift-invariant kernel (i.e., $k(m,n)$ is within some fixed interval for any $m,n$), then $k(m,n)=\exp(-a|m-n|^p)$ with $a>0$ and $p\in (0,2]$ may be of interest.

\paragraph{Constructing Unbounded Shift-invariant Kernels.} A limitation of Fact~\ref{fact:exp} is that it only generates kernels with a bounded range (here, the range is bounded in $(0,1]$). In situations where there are no explicit bounds, one might want to consider kernels with unbounded range. To construct such kernels, we utilize a kernel representation theorem presented in \citet{schoenberg1938metric}:
\begin{fact}[Theorem 4 of \citet{schoenberg1938metric}]
$f(x)$ is bounded away from zero ($f(x)>0$) and its positive powers $f(x)^\lambda$ ($\lambda>0$) are all positive definite if and only if $f(x)$ is of the form
$$f(x)=\exp(c+\psi(x)),$$
where $\psi(x)$ is positive definite and $c$ is a real constant.
\label{fact:represent}
\end{fact}

Since Fact~\ref{fact:represent} works for non-negative kernels, we combine it with Fact~\ref{fact:exp} and show the following class of shift-invariant kernel with an unbounded range.
\begin{proposition}[Kernel from Distance Powers]
	For any $p\in (0,2]$, there exists $c_{\min}\in\mathbb{R}$ such that for any $c\geq c_{\min}$, $$k(m,n)=c-|m-n|^p,$$ 
	is a positive definite kernel. When $p>2$, there is no $c$ to make $k(m,n)$ positive definite.
	\label{prop:dist-ker}
\end{proposition}
\begin{proof}
Due to Fact~\ref{fact:exp}, we know $\exp(-|x|^p)$ is positive definite when $p\in(0,2]$. Since $\exp(-|x|^p)>0$ and $\exp(-\lambda|x|^p)$ is positive definite for any $\lambda >0$\footnote{$\exp(-\lambda|x|^p)=\exp(-|\lambda^{1/p} x|^p)$ is a constant rescaling of $\exp(-|x|^p)$ and therefore is positive definite.}, Fact~\ref{fact:represent} implies there exists a $c'\in\mathbb{R}$ and a positive definite $\psi(x)$ such that
$$\exp(-|x|^p)=\exp(c'+\psi(x)).$$
In other words, $-c'-|x|^p$ is a positive definite function. Take $c_{\min}=-c'$. We see that $c_{\min}-|m-n|^p$ is a shift-invariant positive definite kernel. Finally, let $k(m,n)=c-|m-n|^p$ with $c\geq c_{\min}$. The $N\times N$ matrix $[k(x_i,x_j)]_{i,j=1}^N$ generated by $k(m,n)$ on points $\{x_i\in\mathbb{R}\}_{i=1}^N$ obeys
$$[k(x_i,x_j)]_{i,j=1}^N = [c_{\min}-|x_i-x_j|^p]_{i,j=1}^N + (c-c_{\min})\one\one^\top \succeq (c-c_{\min})\one\one^\top \succeq 0,$$
where $\succeq$ is the Loewner order and $\one = [1,...,1]^\top$ in the $N$-dimensional vector with all ones. This shows $k(m,n)$ is a shift-invariant positive definite kernel when $0<p\leq 2$. The conclusion on $p>2$ is proved by contradiction. When $p>2$, if there exists a $c$ such that $k(m,n)$ is positive definite, then $\exp(k(m,n))$ is positive definite, which contradicts to the case of $p>2$ in Fact~\ref{fact:exp}.
\end{proof}
Prop.~\ref{prop:dist-ker} introduces a kernel with unbounded range ($k(m,n)\in (\infty,c]$) and is adapted from the p-th power of the distance $|m-n|$. Since the distance is a notion of "dissimilarity", $-|m-n|^p$ becomes a notion of "similarity", which gives a sense of kernel. Thereby, we can interpret the constant $c$ as the required value to shift $-|m-n|^p$ such that $c-|m-n|^p$ becomes a positive definite kernel.

In fact, $-|m-n|^p$ is a conditional positive definite kernel for $p\in(0,2]$ \citet{Schplkopf2000cpd}. Therefore, the fact that $-|m-n|^p$ can become a positive definite kernel by shifting is not a coincidence, as it has already had an intimate relation to positive definite kernels.

\subsection{Experiments on Gaussian-like Kernels}
\label{appendix:experi_gauss}
Since the prior work on shift-invariant kernels mainly focuses on Gaussian kernels, we present preliminary experiments on Gaussian-like kernels. Compared with section~\ref{sec:experi_results}, the perplexities of these kernels are large at every extrapolation length. This verifies our previous assertion that the Gaussian-like kernels have limited extrapolation ability.

Because the kernel can be used as a weight or a bias, we consider four kinds of the composite kernel (see section~\ref{sec:krpe}) as follows.
\begin{enumerate}[topsep=-3pt, itemsep=-1pt, leftmargin=22mm]
    \item[(2-para-bias)] $r_1,r_2>0$.\\
    $k^{\text{comp}}([\bm q_m,m],[\bm k_n,n])= \bm q_m^\top \bm k_n + r_1\exp(-r_2|m-n|^2)$.
    \item[(3-para-bias)] $r_1,r_2>0$ and $0<r_3\leq 2$. \\
    $k^{\text{comp}}([\bm q_m,m],[\bm k_n,n])= \bm q_m^\top \bm k_n + r_1\exp(-r_2|m-n|^{r_3})$.
    \item[(1-para-wht)] $r_1>0$.\\
    $k^{\text{comp}}([\bm q_m,m],[\bm k_n,n])= \bm q_m^\top \bm k_n \cdot \exp(-r_1|m-n|^2)$.
    \item[(2-para-wht)] $r_1>0$ and $0<r_2\leq 2$.\\
    $k^{\text{comp}}([\bm q_m,m],[\bm k_n,n])= \bm q_m^\top \bm k_n \cdot \exp(-r_1|m-n|^{r_2})$.
\end{enumerate}
(2-para-bias) and (1-para-wht) are the settings where we put the Gaussian kernel as a bias and a weight, respectively. (3-para-bias) and (2-para-wht) generalize these settings by considering a learnable power between 0 and 2. Note we must constrain the power in (0,2]; otherwise, the function is not positive definite. See Fact~\ref{fact:exp} for details. 

These composite kernel $k^{\text{comp}}$ are plugged into the KERPLE framework, Eq.~\eqref{eq:krpe-deriv}, and are evaluated on OpenWebText2, GitHub, and ArXiv datasets. Table~\ref{tab:gauss-like} shows the Gaussian-like kernel is better to be a weight instead of a bias. As discussed in Appendix \ref{appendix:alternative}, the Gaussian-like kernels are bounded. To some extent, this implies that the bounded positive kernel can model a weight. However, compared with section~\ref{sec:experi_results} the Gaussian-like kernels have limited advantages in extrapolation. Although the performance might be improved if the power of $\exp(-|x|^p)$ is relaxed from $p=2$ to $p\in(0,2]$, still it cannot be as good as the logarithmic variant as we demonstrate in section~\ref{sec:experi_results}. Therefore, while the Gaussian kernel is frequently used in the literature, we need a better class of shift-invariant kernels to tackle the length extrapolation challenge.

\begin{table}[!ht]
    \centering
    \caption{\textbf{Extrapolation of Gaussian-like kernels on OpenWebText2, GitHub, and ArXiv.} All models are trained for 50k steps with training length 512 and five random seeds.}
    \begin{tabular}{@{\extracolsep{3pt}}lcccc}
    \hline\hline
    \multirow{2}{*}{Extrp.} & \multicolumn{4}{c}{OpenWebText2}\\
     \cline{2-5}
     &   1-para-wht & 2-para-wht & 2-para-bias & 3-para-bias \\ \hline
    512 & 33.8 $\pm$ 1.1 & 24.8 $\pm$ 0.9 & 58.4 $\pm$ 71.6 & 26.4 $\pm$ 0.5\\
1024 & 32.5 $\pm$ 0.8 & 23.0 $\pm$ 0.8 & 88.7 $\pm$ 62.6 & 75.3 $\pm$ 37.8\\
2048 & 34.1 $\pm$ 0.6 & 22.7 $\pm$ 0.4 & 406 $\pm$ 101 & 2629 $\pm$ 4024\\
4096 & 35.6 $\pm$ 0.9 & 22.6 $\pm$ 0.6 & 2590 $\pm$ 3211 & 37557 $\pm$ 67936\\
8192 & 39.2 $\pm$ 1.1 & 23.2 $\pm$ 0.3 & 10829 $\pm$ 18855 & 189216 $\pm$ 369499\\
    \hline\hline
    \multirow{2}{*}{Extrp.} & \multicolumn{4}{c}{GitHub}\\
     \cline{2-5}
     &   1-para-wht & 2-para-wht & 2-para-bias & 3-para-bias \\ \hline
    512 & 7.78 $\pm$ 0.48 & 3.56 $\pm$ 0.23 & 4.08 $\pm$ 0.85 & 3.67 $\pm$ 0.22\\
1024 & 7.85 $\pm$ 0.40 & 3.19 $\pm$ 0.17 & 4.63 $\pm$ 0.59 & 4.23 $\pm$ 0.57\\
2048 & 8.08 $\pm$ 0.21 & 3.01 $\pm$ 0.09 & 18.8 $\pm$ 6.8 & 20.0 $\pm$ 4.8\\
4096 & 8.47 $\pm$ 0.43 & 2.93 $\pm$ 0.09 & 75.8 $\pm$ 32.2 & 94.0 $\pm$ 24.7\\
8192 & 9.41 $\pm$ 0.75 & 3.05 $\pm$ 0.20 & 207 $\pm$ 110 & 261 $\pm$ 86\\
    \hline\hline
    \multirow{2}{*}{Extrp.} & \multicolumn{4}{c}{ArXiv}\\
     \cline{2-5}
     &   1-para-wht & 2-para-wht & 2-para-bias & 3-para-bias \\ \hline
    512 & 10.6 $\pm$ 0.4 & 6.18 $\pm$ 0.25 & 6.73 $\pm$ 0.30 & 7.12 $\pm$ 1.43\\
1024 & 10.7 $\pm$ 0.2 & 5.73 $\pm$ 0.11 & 7.07 $\pm$ 0.63 & 7.27 $\pm$ 0.69\\
2048 & 10.8 $\pm$ 0.3 & 5.35 $\pm$ 0.15 & 20.4 $\pm$ 9.3 & 23.5 $\pm$ 8.4\\
4096 & 11.6 $\pm$ 0.3 & 5.44 $\pm$ 0.14 & 80.6 $\pm$ 49.4 & 131 $\pm$ 140\\
8192 & 12.1 $\pm$ 0.2 & 5.50 $\pm$ 0.27 & 220 $\pm$ 138 & 437 $\pm$ 591\\
    \hline\hline
    \end{tabular}
    \label{tab:gauss-like}
\end{table}

\begin{table}[!ht]
    \centering
    \caption{\textbf{Training Time Comparison for Gaussian-like Kernels on GitHub.} }
    \begin{tabular}{lcccc}
    \hline\hline
     %& \multicolumn{2}{c}{KERPLE} \\
     % \cline{2-3}
     & 1-para-wht & 2-para-wht & 2-para-bias & 3-para-bias \\ \hline
    sec/step & 0.326 & 0.327 & 0.324 & 0.351\\
    \hline\hline
    \end{tabular}
    \label{tab:gauss_speed}
\end{table}

\subsection{Experiments on Large Model, Longer Training Length, and Wikitext-103}
\label{appendix:more-experi}

In this subsection, we present additional experiments on (a) large models, (b) longer training length, and (c) Wikitext-103. Below is the summary of the experiments.

\begin{enumerate}[label=(\alph*), topsep=-3pt, itemsep=-1pt, leftmargin=6mm]
\item The 1.3B large model is trained on a machine with two NVIDIA A100 GPU with 40 GB of memory. We adopt almost all configurations of XL GPT-NeoX\footnote{\url{https://github.com/EleutherAI/gpt-neox/blob/main/configs/XL.yml}}, except that we change the train-micro-batch-size to 16, model-parallel-size to 2, seq-length to 512, and max-position-embeddings to 512. Table~\ref{tab:large_model_configs} summarizes the configurations of the 1.3B model.
\item The 162M Model with training sequence length=1024 follows the same configurations as the ones in Table~\ref{tab:model_configs} except that the train seq. length is changed to 1024.
\item The Wikitext-103 model is implemented on ALiBi's GitHub\footnote{\url{https://github.com/ofirpress/attention_with_linear_biases}} with exactly the same configurations (247M parameters), except that the function buffered\_future\_mask() at line 1011 of attention\_with\_linear\_biases/fairseq/models/transformer.py is adapted to our KERPLE-log. 
\end{enumerate}
\begin{table}[!ht]
    \centering
    \setlength{\tabcolsep}{3pt}
    \caption{\textbf{1.3B Model Configurations.}}
    \begin{tabular}{ccccc}
        \hline\hline
         \# Layers & Hidden Size & \# Attention Heads & Train Seq. Len. & \# Trainable Params.\\
         24 & 128 & 16 & 512 & ~1.3B\\ \hline
         Optimizer & Batch Size & Train Steps & Precision & \# Trainable Params. for RPEs\\
         Adam (lr 2e-4) & 32 & 150,000 & float16 & 48\\
         \hline\hline
    \end{tabular}
    \label{tab:large_model_configs}
\end{table}

Table~\ref{tab:large-scale-and-long-length} shows the results on the large model (1.3B). Compared with the small model results in Table~\ref{tab:openweb-github-arxiv}, we see that T5 bias becomes weaker than KERPLE-log and ALiBi, and KERPLE-log remains stronger than ALiBi on GitHub and ArXiv datasets. This is explained by the tendency of overfitting. Observe that both T5 and KERPLE learn the positional embeddings while ALiBi uses fixed ones. T5 and KERPLE have a higher tendency of overfitting. A larger model (1.3B > 162M) or a noisy dataset (OpenWebText2 > GitHub, ArXiv) posits a higher risk of overfitting. Hence, we see that T5 bias is weak on a large model, and KERPLE-log only extrapolates well on GitHub and ArXiv.

%This makes sense because T5 is more susceptible to overfitting and the large model has a higher risk of overfitting. 

Again, Table~\ref{tab:large-scale-and-long-length} shows the results on long training length (1024). compared with the short training length (512) in Table~\ref{tab:openweb-github-arxiv}, KERPLE-log remains better than ALiBi and T5 bias, especially on longer evaluation length. This shows the robustness of KERPLE-log over different training lengths.

Table~\ref{tab:wikitext103} compares KERPLE-log with ALiBi using ALiBi's implementation and configurations. The results show that KERPLE-log is superior to ALiBi on Wikitext-103.

\begin{table}[!ht]
    \setlength{\tabcolsep}{2pt}
    \centering
    \caption{\textbf{Perplexity Comparison for Large Models (1.3B) and Long Training Length (1024) on GitHub, ArXiv, OpenWebText2.} Due to the time constraint and limited computing resources, we are not able to obtain the numbers for the large model (1.3B) on OpenWebText2 for now. All models are trained with five random seeds. $x^\dagger$ means our log variant is statistically significantly~\emph{better} than $x$. The test used is paired two-sided t-test with $\alpha=0.05$.}
    \begin{tabular}[t]{@{\extracolsep{3pt}}lccc}
    \hline\hline
    \multicolumn{4}{c}{162M Model. Train length, steps=1024, 50k.}\\
    \hline\hline
    \multirow{2}{*}{Extrp.} & \multicolumn{3}{c}{GitHub}\\
     \cline{2-4}
     &   KERPLE-log & ALiBi & T5 bias \\ \hline
     512 & - & - & -\\
    1024 & 2.83 $\pm$ 0.16 & 2.84 $\pm$ 0.16$^\dagger$ & \textbf{2.81 $\pm$ 0.16}\\
2048 & 2.70 $\pm$ 0.07 & 2.82 $\pm$ 0.07$^\dagger$ & \textbf{2.68 $\pm$ 0.07}\\
4096 & \textbf{2.53 $\pm$ 0.04} & 2.77 $\pm$ 0.06$^\dagger$ & 2.54 $\pm$ 0.04\\
8192 & \textbf{2.42 $\pm$ 0.03} & 2.74 $\pm$ 0.02$^\dagger$ & 2.57 $\pm$ 0.06$^\dagger$\\
16384 & \textbf{2.48 $\pm$ 0.11} & 2.80 $\pm$ 0.11$^\dagger$ & 3.10 $\pm$ 0.34$^\dagger$\\
    \hline\hline
    \multirow{2}{*}{Extrp.} & \multicolumn{3}{c}{ArXiv}\\
     \cline{2-4}
     &   KERPLE-log & ALiBi & T5 bias \\ \hline
    512 & - & - & -\\
1024 & 5.23 $\pm$ 0.09 & 5.26 $\pm$ 0.09 & \textbf{5.20 $\pm$ 0.10}\\
2048 & 4.76 $\pm$ 0.12 & 4.98 $\pm$ 0.18$^\dagger$ & \textbf{4.74 $\pm$ 0.12}\\
4096 & \textbf{4.75 $\pm$ 0.10} & 5.31 $\pm$ 0.13$^\dagger$ & 4.97 $\pm$ 0.27\\
8192 & \textbf{4.54 $\pm$ 0.10} & 5.25 $\pm$ 0.15$^\dagger$ & 6.55 $\pm$ 0.97$^\dagger$\\
16384 & \textbf{4.62 $\pm$ 0.15} & 5.35 $\pm$ 0.19$^\dagger$ & 16.0 $\pm$ 4.77$^\dagger$\\
    \hline\hline
    \multirow{2}{*}{Extrp.} & \multicolumn{3}{c}{OpenWebText2}\\
     \cline{2-4}
     & KERPLE-log & ALiBi & T5 bias \\ \hline
512 &    -  & - & -\\
1024 & 19.2 $\pm$ 0.1 & 19.3 $\pm$ 0.2 & \textbf{19.1 $\pm$ 0.1}\\
2048 & 19.3 $\pm$ 0.2 & 19.5 $\pm$ 0.1 & \textbf{19.2 $\pm$ 0.2}\\
4096 & \textbf{18.6 $\pm$ 0.3} & 19.0 $\pm$ 0.3$^\dagger$ & 19.2 $\pm$ 0.4$^\dagger$\\
8192 & \textbf{18.7 $\pm$ 0.5} & 19.3 $\pm$ 0.4$^\dagger$ & 24.0 $\pm$ 1.1$^\dagger$\\
16384 & \textbf{18.8 $\pm$ 0.5} & 19.2 $\pm$ 0.3$^\dagger$ & 50.8 $\pm$ 6.5$^\dagger$\\
\hline\hline
    \end{tabular}
    \begin{tabular}[t]{ccc}
    \hline\hline
    \multicolumn{3}{c}{1.3B Model. Train length, steps=512, 150k.}\\
    \hline\hline
\multicolumn{3}{c}{GitHub}\\
     \cline{1-3}
KERPLE-log & ALiBi & T5 bias \\ \hline
\textbf{2.88 $\pm$ 0.11} & \textbf{2.88 $\pm$ 0.11} & 2.93 $\pm$ 0.11$^\dagger$\\
\textbf{2.60 $\pm$ 0.12} & 2.62 $\pm$ 0.11$^\dagger$ & 2.64 $\pm$ 0.11$^\dagger$\\
\textbf{2.44 $\pm$ 0.05} & 2.58 $\pm$ 0.05$^\dagger$ & 2.47 $\pm$ 0.07$^\dagger$\\
\textbf{2.46 $\pm$ 0.11} & 2.65 $\pm$ 0.12$^\dagger$ & 2.49 $\pm$ 0.12\\
\textbf{2.44 $\pm$ 0.13} & 2.57 $\pm$ 0.13$^\dagger$ & 2.57 $\pm$ 0.13$^\dagger$\\
\textbf{2.60 $\pm$ 0.07} & 2.61 $\pm$ 0.07 & 3.16 $\pm$ 0.35$^\dagger$\\
    \hline\hline
\multicolumn{3}{c}{ArXiv}\\
     \cline{1-3}
KERPLE-log & ALiBi & T5 bias \\ \hline
\textbf{5.56 $\pm$ 0.15} & 5.58 $\pm$ 0.16 & 5.62 $\pm$ 0.15$^\dagger$\\
\textbf{4.87 $\pm$ 0.07} & 4.94 $\pm$ 0.07$^\dagger$ & 4.92 $\pm$ 0.06$^\dagger$\\
\textbf{4.50 $\pm$ 0.16} & 4.87 $\pm$ 0.17$^\dagger$ & 4.55 $\pm$ 0.16$^\dagger$\\
\textbf{4.45 $\pm$ 0.06} & 4.97 $\pm$ 0.13$^\dagger$ & 4.53 $\pm$ 0.08$^\dagger$\\
\textbf{4.47 $\pm$ 0.20} & 4.94 $\pm$ 0.16$^\dagger$ & 4.65 $\pm$ 0.15$^\dagger$\\
\textbf{4.65 $\pm$ 0.24} & 4.94 $\pm$ 0.07 & 5.25 $\pm$ 0.26$^\dagger$\\
    \hline\hline
    \multicolumn{3}{c}{OpenWebText2}\\
     \cline{1-3}
KERPLE-log & ALiBi & T5 bias \\ \hline
\textbf{17.5 $\pm$ 0.3} & 17.5 $\pm$ 0.4 & 17.8 $\pm$ 0.3$^\dagger$\\
\textbf{16.6 $\pm$ 0.6} & 16.7 $\pm$ 0.6 & 16.9 $\pm$ 0.6$^\dagger$\\
\textbf{16.2 $\pm$ 0.4} & 16.4 $\pm$ 0.4$^\dagger$ & 16.7 $\pm$ 0.4$^\dagger$\\
\textbf{16.4 $\pm$ 0.8} & 16.5 $\pm$ 0.5 & 18.0 $\pm$ 0.9$^\dagger$\\
16.9 $\pm$ 0.7 & \textbf{16.5 $\pm$ 0.1} & 22.7 $\pm$ 3.7$^\dagger$\\
17.8 $\pm$ 1.2 & \textbf{16.5 $\pm$ 0.3} & 37.1 $\pm$ 13.1$^\dagger$\\
    \hline\hline
    \end{tabular}
    \label{tab:large-scale-and-long-length}
\end{table}

\begin{table}[!ht]
    \caption{\textbf{Perplexity Comparison on Wikitext-103.} To ensure a fair comparison, the model (247M) is trained on ALiBi's codebase with exactly the same configurations except for the positional embeddings. The results show that KERPLE-log is superior to ALiBi on Wikitext-103.}
    \centering
    \begin{tabular}{lccccc|cc}
    \hline\hline
      & \multicolumn{5}{c}{train length 512}&\multicolumn{2}{c}{train length 2048}\\
     Extrp. length & 512 & 1024 & 1536 & 2048 & 3072 & 2048 & 3072\\
     \hline
     ALiBi    &  19.73 & 18.81 & 18.50 & 18.48 & 18.40 & 17.91& 17.64\\
     KERPLE-log    &  \textbf{19.69} & \textbf{18.76} & \textbf{18.37} & \textbf{18.29} & \textbf{18.24} & \textbf{17.84} & \textbf{17.56}\\
     \hline\hline
    \end{tabular}
    \label{tab:wikitext103}
\end{table}

\subsection{Additional Analyses}
\label{sec:analysis}
Since the power and logarithmic variants derived from KERPLE achieve superior performance on length extrapolation across various datasets, we investigate the underlying reason by visualizing the \emph{effective length} as shown in Figure~\ref{fig:cutoff}. The visualization works in the following procedure.
\begin{enumerate}[leftmargin=5mm]
    \item For each training dataset, the learnable parameters $(r_1^{(h)},...,r_\ell^{(h)})$ associated with each head $h$ (12 in total) are extracted from the model checkpoint. The CPD kernel at head $h$ is $\tilde{k}^{(h)}=\tilde{k}_{r_1^{(h)},...,r_\ell^{(h)}}.$
    Both the power and the logarithmic variants in corollary \ref{cor:cpd_examples} undergo a similar procedure. The only difference is that their $\tilde{k}$'s are different.
    
    \item For each head $h$, we compute the~\emph{effective length} of $\tilde{k}^{(h)}$ as $\text{eff}^{(h)}=\underset{\tilde{k}^{(h)}(0,|m-n|)<-2}{\min} |m-n|$. That is, the relative positional difference $|m-n|$ such that $\tilde{k}(m,n)\overset{\text{shift-inv.}}{=}\tilde{k}(0,|m-n|)$ just becomes smaller than -2. Note $\tilde{k}(0,|m-n|)$ strictly decreases in $|m-n|$, so there is only one possible value. We pick $-2$ here because $\tilde{k}$ is a bias and is followed by the Softmax normalization. A bias of $-2$ or smaller can make a great impact on the output of Softmax\footnote{Since Softmax is an exponentiated function, a -2 bias in the Softmax's argument roughly gives an attenuation of $\exp(-2)\approx 0.135$.}. $\text{eff}^{(h)}$ is interpreted as the \emph{effective length} because, when $|m-n|<\text{eff}^{(h)}$, the attenuation due to $\tilde{k}^{(h)}$ is not strong. When $|m-n|>\text{eff}^{(h)}$, the attenuation is strong and has a large impact on $\bm q_m^\top \bm k_n+\tilde{k}^{(h)}(m,n)$.
    \item Then, for each $|m-n|\in[0,...,20480]$, we count the number of heads that satisfies $\text{eff}^{(h)}\leq |m-n|$. This gives a cumulative plot as shown in Figure~\ref{fig:cutoff}, where the x-axis is $|m-n|$ and the y-axis is $\text{Count}(\{h:~h\in[1,...,12],~\text{eff}^{(h)}\leq |m-n|\})$.
    \item Repeat the above steps for other datasets and kernels.
\end{enumerate}
\paragraph{Interpretation of Curves.}
For a point $(x,y)$ on a curve, it means that there are $y$ heads with at least $-2$ bias when the token distance $|m-n|$ is greater than $x$. In other words, the slower the $y$ converges to 12, the longer the inter-token range that the model focuses on.

\paragraph{The Advantage of Learnable Parameters.}
We observe that ALiBi \citep{press2022train} produces the same curve no matter which dataset is used. The reason is that ALiBi selects a fixed parameter $r=2^{\frac{-8h}{H}}$ at head $h$ for its linear bias $-r|m-n|$ ($H$ heads in total) regardless of the dataset. While this strategy is useful for extrapolation, we hypothesize that different datasets might have different characteristics, e.g., the average distance of highly related tokens should differ among the datasets as shown in Figure~\ref{fig:cutoff}. These characteristics are easier adapted by learnable parameters. Thus, we believe that learnable parameters have more advantages in capturing the dataset-dependent characteristics.

\begin{figure*}[t]
    \centering
    \caption{\textbf{Number of Heads with Effective Lengths $\pmb{\leq|m-n|}$ for different choices of CPD kernels and datasets.} See section~\ref{sec:analysis} for details.}%
    \subfloat[][\centering Power Variant: $-a|m-n|^p$]{{\includegraphics[width=0.4\textwidth]{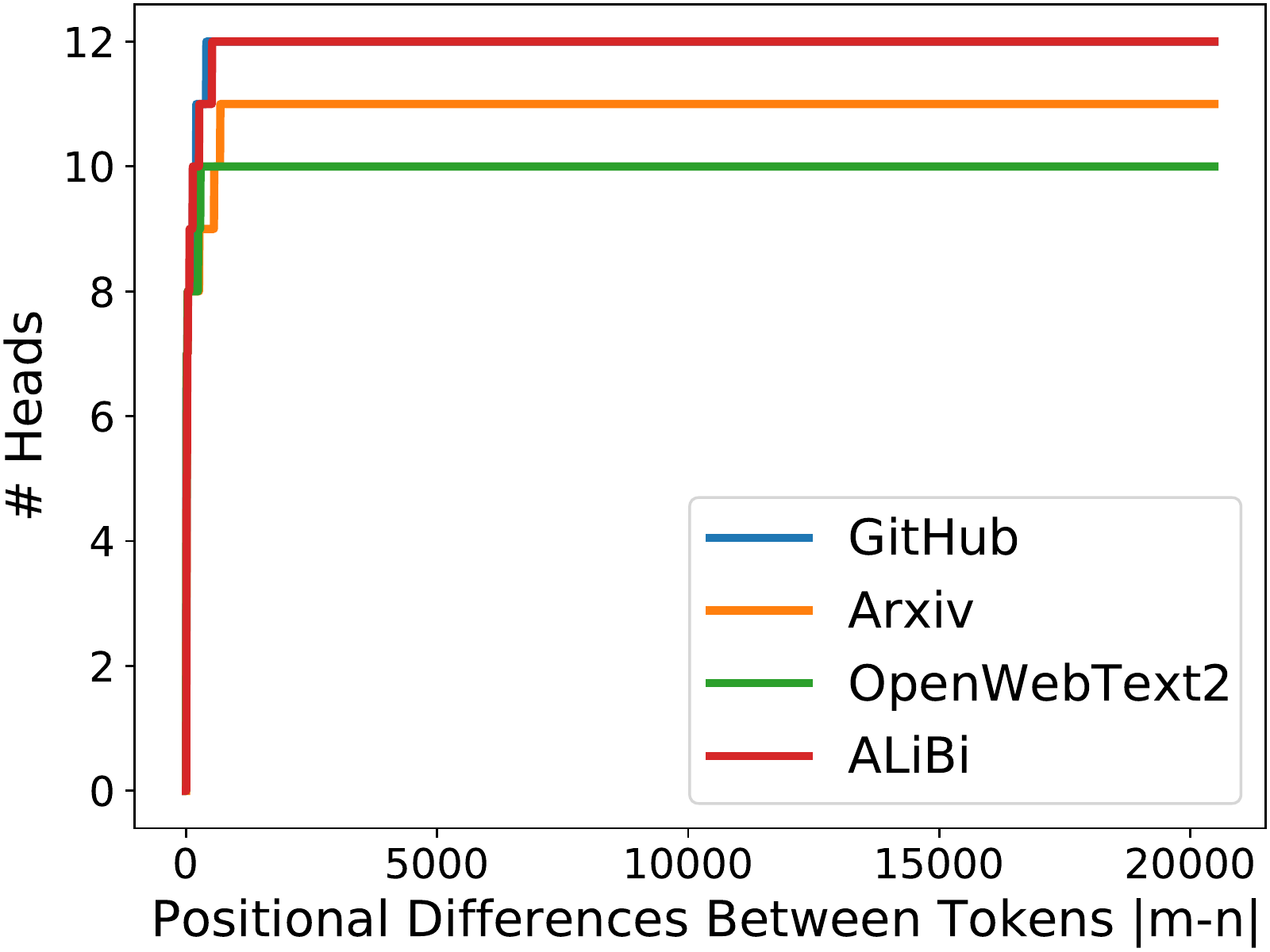} }}%
    \qquad\qquad
    \subfloat[][\centering Logarithmic Variant: $-a\log(1+b|m-n|)$]{{\includegraphics[width=0.4\textwidth]{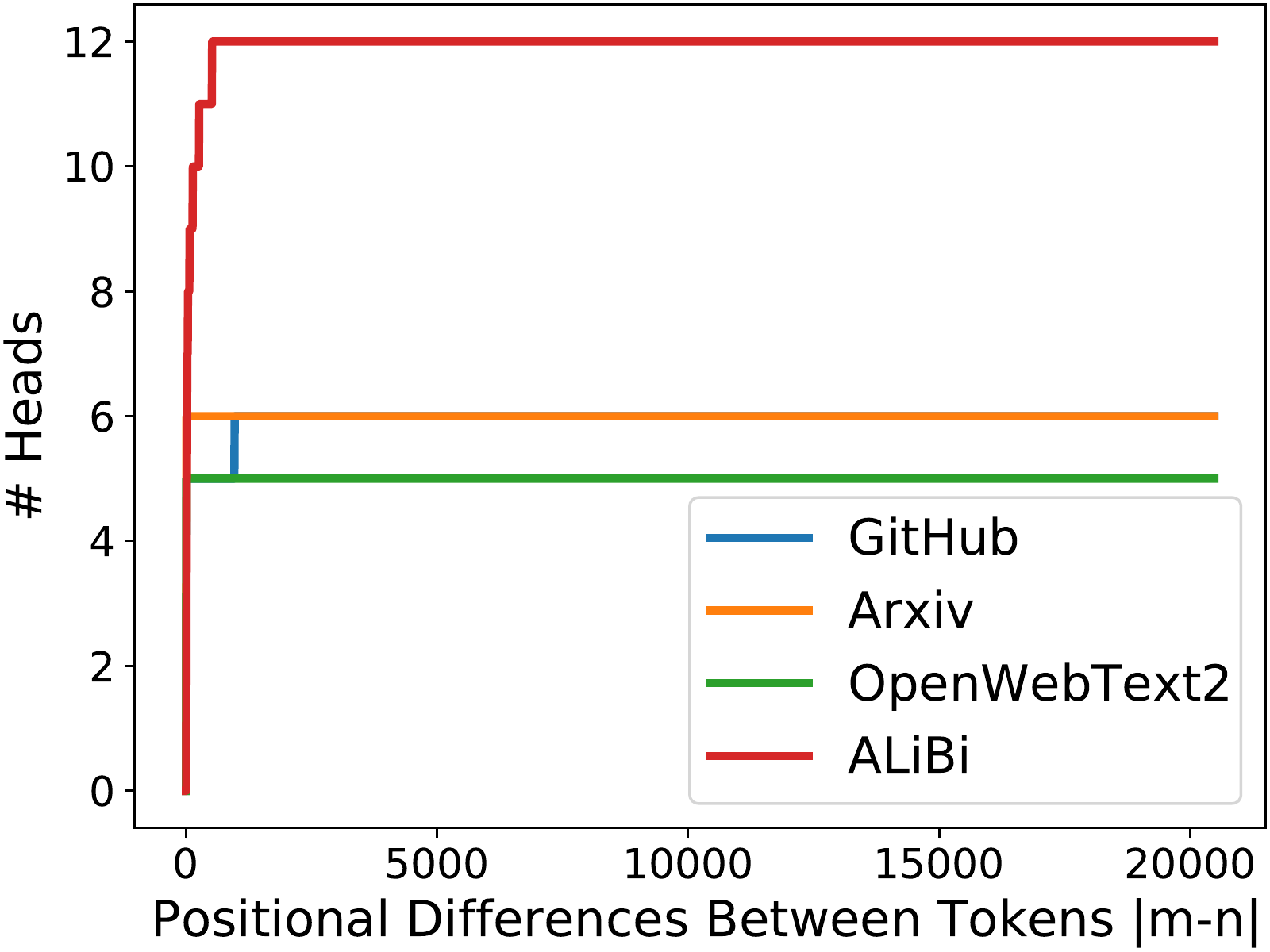} }}%
    \label{fig:cutoff}%
    \vspace{-7mm}
\end{figure*}

\paragraph{Trends Across Datasets.}
We notice that both kernels trained on OpenWebText2 tend to focus more on distant relations. This makes sense because OpenWebText2 has the highest perplexity scores among all datasets, implying that more context is needed to disambiguate the next predicted token. The opposite trend holds for Arxiv and GitHub datasets, which is reasonable considering their lower perplexity scores.

\paragraph{Characteristics Learned by Kernels.}
Under any dataset, the logarithmic variant tends to focus more on distant relations than the power variant does. We can explain it through their functional forms. Because logarithm ($-a\log(1+b|m-n|)$) decays much slower than power ($-a|m-n|^p$) does, the log variant might encourage the focus on distant relations.

\subsection{Position-wise Perplexity for Length=16384}
\label{sec:posppl_16384}
We can draw similar conclusions from Figure~\ref{fig:pos-wise-ppl-window-16384}:
\begin{enumerate}
    \item KERPLE-log lies below KERPLE-log-windowed@512 most of the time, indicating its usage of more distant information than window attention.
    \item The PPL of T5 explodes.
    \item The PPL of ALiBi does not explode, but it is still worse than window attention, i.e. lies above KERPLE-log-windowed@512.
\end{enumerate}
\begin{figure}[!ht]
    \centering   
    \caption{\textbf{Position-wise Perplexity on GitHub at Evaluation Length=16384 Compared to Window Attention@512.}}
    {{\includegraphics[width=\textwidth]{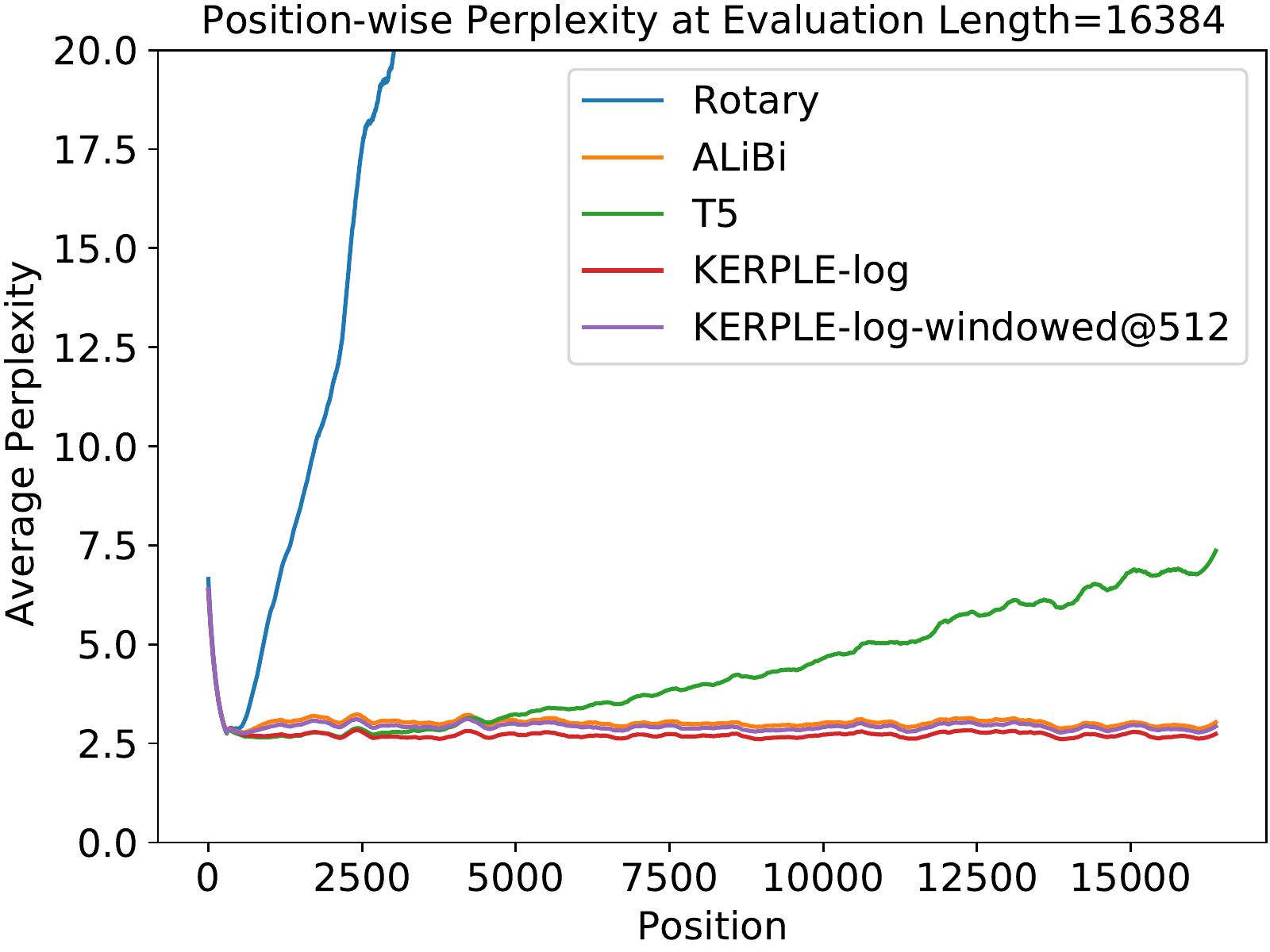} }}%
    \label{fig:pos-wise-ppl-window-16384}
\end{figure}

\subsection{The Choice of codebase and Hyperparameters}
\label{sec:hyper_model}
We adopt almost all the hyperparameters (except batch size to fit in our GPU) and all implementations of the T5 bias, ALiBi, Rotary, and Sinusoidal baselines from the GPT-NeoX codebase. To ensure fair comparisons, we did not fine-tune hyper-parameters for KERPLE.
The datasets we used are exactly the same as the ones released with the GPT-NeoX codebase. We just ran their prepare\_data.py script to automatically download and parse the datasets. All our code was uploaded with the submission on openreview, and \url{https://github.com/EleutherAI/gpt-neox} is the original GitHub repository.
As a side note, we chose this codebase and adopted their parameter settings because it is built by EleutherAI, which is a well-known and truly non-profit group of researchers publishing various famous pretrained models for academia including GPT-J-6B and GPT-NeoX-20B.

\end{document}